\documentclass[12pt]{article}

\usepackage{amsmath}
\usepackage{amssymb}
\usepackage{amsthm}
\usepackage{bbm}
\usepackage{bm}
\usepackage{cancel}
\usepackage[]{enumitem}
\usepackage[pdfnewwindow]{hyperref}
\usepackage{longtable}
\usepackage{marginnote}
\usepackage{mathtools}
\usepackage{multirow}
\usepackage{natbib}
\usepackage{physics}
\usepackage{soul}
\usepackage{caption,subcaption}
\usepackage{tcolorbox}
\usepackage{tikz}
\usepackage{url} 
\usepackage[a4paper, total={6in, 8.5in}]{geometry}

\allowdisplaybreaks

\newcommand{\N}{\mathbb{N}}

\newcommand{\R}{\mathbb{R}}

\newcommand{\defeq}{\vcentcolon=}
\newcommand{\eqdef}{=\vcentcolon}

\newtheorem{theorem}{Theorem}[section]
\newtheorem{defn}[theorem]{Definition}
\newtheorem{lem}[theorem]{Lemma}
\newtheorem{prop}[theorem]{Proposition}
\newtheorem{cor}[theorem]{Corollary}

\newtheorem{rk}[theorem]{Remark}

\providecommand{\keywords}[1]
{
  \small	
  \textbf{\textit{Keywords---}} #1
}

\usepackage{lineno}

\title{Nonclosedness of Sets of Neural Networks in Sobolev Spaces}
\date{}

\author{
	Scott Mahan\thanks{Department of Mathematics, University of California San Diego. 9500 Gilman Drive \#0112, La Jolla, CA 92037.} \footnote{Corresponding author. E-mail: scmahan@ucsd.edu} 
	\and
	Emily J.\ King\thanks{Department of Mathematics, Colorado State University. 1874 Campus Delivery, Fort Collins, CO 80523.} 
	\and
	Alex Cloninger\footnotemark[1] \thanks{Halicioglu Data Science Institute, University of California San Diego. 9500 Gilman Drive \#0112, La Jolla, CA 92037.}
}

\begin{document}
\maketitle

\begin{abstract}
	We examine the closedness of sets of realized neural networks of a fixed architecture in Sobolev spaces. For an exactly $m$-times differentiable activation function $\rho$, we construct a sequence of neural networks $(\Phi_n)_{n \in \N}$ whose realizations converge in order-$(m-1)$ Sobolev norm to a function that cannot be realized exactly by a neural network. Thus, sets of realized neural networks are not closed in order-$(m-1)$ Sobolev spaces $W^{m-1,p}$ for $p \in [1,\infty]$. We further show that these sets are not closed in $W^{m,p}$ under slightly stronger conditions on the $m$-th derivative of $\rho$. 
	For a real analytic activation function, we show that sets of realized neural networks are not closed in $W^{k,p}$ for \textit{any} $k \in \N$. 
	The nonclosedness allows for approximation of non-network target functions with unbounded parameter growth.
	We partially characterize the rate of parameter growth for most activation functions by showing that a specific sequence of realized neural networks can approximate the activation function's derivative with weights increasing inversely proportional to the $L^p$ approximation error.
	Finally, we present experimental results showing that networks are capable of closely approximating non-network target functions with increasing parameters via training.
	
\end{abstract}

\keywords{Fixed-architecture neural networks, Neural network expressivity, Closedness, Sobolev space}

\section{Introduction} \label{sec:intro}




From an approximation theory perspective, neural networks use observed training data to approximate an unknown target function. 
Studying topological properties of sets of neural networks will reveal what kinds of functions can be approximated by neural networks. In particular, closedness of sets of networks is a topological property of interest. If these sets are closed with respect to some norm, then one can construct a sequence of neural networks converging to a target function in that norm if and only if that target function is itself a network. On the other hand, nonclosedness would mean that neural networks can approximate target functions that are not networks themselves. An up-to-date survey on the approximation results for neural networks can be found in \citep{GRK20}. 

To allow neural networks to approximate a wider class of functions, the number of nodes in the network can be increased. As long as the number of hidden nodes is allowed to grow without bound, Hornik's Universal Approximation Theorem shows that neural networks with only one hidden layer can approximate any $p$-integrable function to arbitrary accuracy \citep{H91}. Other approximation theorems show that neural networks are dense in other function classes, depending on the properties of the activation function, but many of these results allow the depth or width of the network to vary \citep{C89,HSW89,HSW90,KSH12}. These results suggest that sets of realized neural networks are not closed in the corresponding function spaces, since not all of these functions can be represented exactly by a neural network. However, these are results about sets of networks of any width. 



In practice, the architecture of a neural network is fixed before the learning process begins. Hence, we consider properties of sets of neural networks with a fixed architecture. Related to closedness is the best approximation property in $L^p$ spaces, which holds if every $f \in L^p$ has at least one realized neural network $g$ such that $\|f-g\|_{L^p}$ is minimized over all possible networks with the same architecture. In disproving the best approximation property for sigmoidal neural networks of a fixed size, \citep{GP90} show that sets of realized neural networks with the sigmoid activation function are not closed in $L^p$ spaces and claim that this should be true for all nonlinear activation functions. However, \citep{K95} proves that sets of networks with the Heaviside activation function are closed in $L^p$ spaces, and \citep{KKV20} proves a similar result for generalizations of Heaviside-type networks. Moreover, \citep{KKV00} shows that Heaviside perceptron networks do have the best approximation property, but that best approximation maps with these networks are not unique or continuous. More generally, no best approximation maps using fixed-architecture neural networks are continuous \citep{KKV99}.

Petersen, Raslan, and Voigtlaender discuss the closedness and other topological properties of sets of realized neural networks \citep{PRV19,PRV20}. Among other results, they prove that for most commonly used activation functions, these sets are not closed with respect to $L^p$ norms. However, sets of all neural networks with a fixed architecture and uniformly bounded parameters \textit{are} closed, indicating that learning a non-network target function requires the parameters to grow without bound. For example, learning a target function that has fewer derivatives than the activation function requires at least one network parameter to tend to infinity even as the network width increases, as seen in \citep{M97}. We partially investigate the relationship between $L^p$ approximation error and parameter size, finding that the two are approximately inversely proportional when learning the derivative of the activation function.

In \citep{PRV19} (a shortened conference paper version of \citep{PRV20}), the authors speculate that sets of neural networks may be closed in Sobolev spaces, where convergence is stronger. However, we show that this is not true, extending their \textit{nonclosedness} results to convergence in Sobolev norm under additional smoothness assumptions on the activation function using proof techniques similar to those in \citep{PRV20}. Our results apply to the case $p=\infty$ after minor modifications, so sets of realized neural networks are also not closed in $W^{m,\infty}$. 


In some cases, such as network compression or distillation, we have data about the derivatives of the target function in addition to the training data. In these instances, one can train a network to learn the target function and its derivatives. This approach, introduced in \citep{COJSP17} as Sobolev training, often requires less training data and performs better on testing data. Hence, it is natural to consider the theoretical properties of neural networks in Sobolev spaces, as we do in this paper. We also provide some experimental results using Sobolev training with activation functions of varying degrees of smoothness, 
where we are able to approximate non-network target functions in Sobolev norm. This result indicates that sets of realized neural networks are indeed not closed in Sobolev spaces, but also that Sobolev training does not prevent parameter growth and allows us to approximate functions on the boundary of these sets of realized networks. Our experiments exhibit slow parameter growth relative to a fast decrease in approximation error, but there may be target functions that require much faster parameter growth to approximate. 

\subsection{Contributions of this Work} \label{subsec:contributions}


Our work considers sets of realized neural networks with a fixed architecture and a fixed activation function. In particular, we study the closedness of these sets of realized neural networks in Sobolev spaces. Our main contributions are:

\begin{enumerate}[label=\arabic*)]
    \item We establish in Theorem \ref{thm:m-(m+1)} that for an $m$-times differentiable activation function that is not $(m+1)$-times differentiable, sets of realized neural networks are \textit{not} closed in order-$(m-1)$ Sobolev spaces $W^{m-1,p}$ for $p \in [1,\infty]$. We prove this result by constructing a sequence of neural networks that converges in Sobolev norm to a target function that is not a neural network, which follows the approach in \citep{PRV20}.
    
    \item We extend the nonclosedness result of Theorem \ref{thm:m-(m+1)} to $W^{m,p}$ under an additional assumption on the activation function.
    
    \item For real analytic activation functions, Theorem \ref{thm:rho_smooth} shows that sets of realized neural networks are not closed in any order Sobolev spaces.
    
    \item We show analytically in Proposition~\ref{prop:rates} that for most activation functions, the $L^p$ approximation error decays inversely proportional to the growth of network parameters for a given sequence of networks approximating the derivative of the activation function.
    The relationship between approximation error and weight growth relates to Theorem 1 in \citep{M97}, although that result holds for $L^\infty$ error of shallow networks with increasing width.
    
    \item We present some experiments in Section \ref{sec:experiments} demonstrating that neural networks can approximate target functions that require increasingly large parameters. Our example achieves a fast decay in approximation error with a relatively slow growth in the network parameters, which may not be the case for other non-network target functions. These results appear to be robust to varying degrees of smoothness of the activation function and to certain classes of target functions. 
\end{enumerate}

Our nonclosedness results all indicate that neural networks can be trained to approximate non-network target functions in Sobolev norm. However, we will see that doing so will necessarily cause unbounded growth of network parameters. Thus, the training process may be difficult in practice, or regularization techniques may prevent a network from approximating a non-network target function. In our experiments, we train sequences of networks to approximate non-network target functions in Sobolev norm. The networks are able to closely approximate these target functions, providing further evidence that sets of realized neural networks are not closed in Sobolev spaces. Moreover, the ability to numerically approximate functions on the boundary of these sets of realized neural networks speaks to the expressiveness of neural networks in practice. 


\subsection{Outline of this Paper} \label{subsec:outline}

Our work first provides background material on neural networks, then discusses the closedness of realized neural networks in Sobolev spaces, and finally presents some related numerical results. Section \ref{sec:notation} lays out the definitions and notation required for the rest of the paper. In Section \ref{sec:nonclosedness} we begin with our main results that sets of realized neural networks are not closed in Sobolev spaces under reasonable conditions on the activation function. On the other hand, Section \ref{subsec:closedness} studies realizations of networks with bounded parameters, and presents a result that these sets of realizations \textit{are} closed in Sobolev spaces. Section \ref{subsec:rates} analyzes the relationship between $L^p$ approximation error and parameter growth for a network learning its activation function's derivative. We provide experimental results in Section \ref{sec:experiments} that demonstrate the nonclosedness of sets of realized neural networks and show that some classes of non-network target functions can indeed be approximated in Sobolev norm by a sequence of networks with increasing parameters. 


\section{Notation and Definitions} \label{sec:notation}

We first define neural networks. Every network has an architecture which specifies the input dimension, the number of layers, and the number of nodes in each layer. In addition, the network consists of matrix-vector pairs that determine the affine transformation between consecutive layers.

\begin{defn}
	\citep{PRV20} Let $d,L \in \N$. A \textbf{neural network $\Phi$ with input dimension $d$ and $L$ layers} is a sequence of matrix-vector pairs
	\begin{equation*}
		\Phi = \big( (A_1,b_1),\dots,(A_L,b_L) \big),
	\end{equation*}
	where $N_0=d$ and $N_1,\dots,N_L \in \N$, and where each $A_\ell$ is an $N_\ell \times N_{\ell-1}$ matrix, and $b_\ell \in \R^{N_\ell}$. We call $(d,N_1,\dots,N_L)$ the \textbf{architecture} of $\Phi$. $N_L$ is the \textbf{output dimension}. Define $\mathcal{NN}(d,N_1,\dots,N_L)$ to be the \textbf{set of all neural networks} $\Phi$ with architecture $(d,N_1,\dots,N_L)$. 
\end{defn}

To emphasize the role of the activation function, we distinguish between a neural network and a realized neural network. A realized network is a function defined by alternately applying the affine transformations of the network and the activation function. We also define the set of all realized neural networks with a fixed architecture and the same activation function.

\begin{defn}
	\citep{PRV20} Let $\Phi$ be a neural network, $\Omega \subset \R^d$, and $\rho : \R \to \R$. The \textbf{realization of $\Phi$ with activation function $\rho$ over $\Omega$} is the function $R_\rho^\Omega(\Phi) : \Omega \to \R^{N_L}$ defined by
	\begin{equation*}
		R_\rho^\Omega(\Phi)(x) = W_L ( \rho(W_{L-1} ( \cdots \rho( W_1(x)))))
	\end{equation*}
	where the affine transformation $W_\ell : \R^{N_{\ell-1}} \to \R^{N_\ell}$ is defined by $W_\ell(x) = A_\ell x + b_\ell$ and $\rho$ is evaluated componentwise. Define $R_\rho^\Omega$ to be the \textbf{realization map} $\Phi \mapsto R_\rho^\Omega(\Phi)$, and let
	\begin{equation*}
		\mathcal{RNN}_\rho^\Omega(d,N_1,\dots,N_L) \defeq R_\rho^\Omega\big( \mathcal{NN}(d,N_1,\dots,N_L) \big).
	\end{equation*}
	We call $\mathcal{RNN}_\rho^\Omega(d,N_1,\dots,N_L)$ the \textbf{set of $\rho$-realizations of networks with architecture $(d,N_1,\dots,N_L)$ over $\Omega$}. 
\end{defn}

We will sometimes need to concatenate networks, which creates a new neural network consisting of the matrix-vector pairs of the first network followed by the pairs of the second network.

\begin{defn}
	\citep{PRV20} Let $\Phi_1 = \big( (A_1^1,b_1^1),\dots,(A_{L_1}^1,b_{L_1}^1) \big)$ and $\Phi_2 = \big( (A_1^2,b_1^2),\dots,(A_{L_2}^2,b_{L_2}^2) \big)$ be two neural networks such that the input dimension of $\Phi_1$ equals the output dimension of $\Phi_2$. Then
	\begin{equation*}
		\Phi_1 \bullet \Phi_2 \defeq \big( (A_1^2,b_1^2),\dots,(A_{L_2-1}^2,b_{L_2-1}^2),(A_1^1 A_{L_2}^2, A_1^1b_{L_2}^2 + b_1^1), (A_2^1,b_2^1),\dots,(A_{L_1}^1,b_{L_1}^1) \big)
	\end{equation*}
	defines a neural network with $L_1+L_2-1$ layers. We call $\Phi_1 \bullet \Phi_2$ the \textbf{concatenation of $\Phi_1$ and $\Phi_2$}.
\end{defn}
Note that for any activation function $\rho: \R \to \R$ and any $\Omega \subset \R^{d_2}$, we have $R_\rho^\Omega(\Phi_1 \bullet \Phi_2) = R_\rho^{\R^{d_1}}(\Phi_1) \circ R_\rho^\Omega(\Phi_2)$, where $d_i$ is the input dimension of $\Phi_i$. That is, concatenation of neural networks corresponds to function composition of the realizations of those networks.

For a fixed network architecture $(d,N_1,\dots,N_L)$, we want to consider the closedness of the set $\mathcal{RNN}_\rho^\Omega(d,N_1,\dots,N_L)$ in Sobolev spaces. We define Sobolev spaces below.

\begin{defn} 
	Let $k \in \N$, let $\Omega \subset \R^k$ be measurable with non-empty interior, and let $1 \leq p \leq \infty$. The \textbf{Sobolev space} $W^{k,p}(\Omega)$ consists of all functions $f$ on $\Omega$ such that for all multi-indices $\alpha$ with $|\alpha|\leq k$, the mixed partial derivative $f^{(\alpha)} \defeq D^\alpha f$ exists in the weak sense and belongs to $L^p(\Omega)$. That is,
	\begin{equation*}
		W^{k,p}(\Omega) = \left\{ f \in L^p(\Omega) : D^\alpha f \in L^p(\Omega) \text{ for all } |\alpha| \leq k \right\}.
	\end{equation*}
	The number $k$ is the \textbf{order} of the Sobolev space. The norm
	\begin{equation*}
		\|f\|_{W^{k,p}(\Omega)} \defeq \sum_{|\alpha| \leq k} \| D^\alpha f \|_{L^p(\Omega)}
	\end{equation*}
	makes $W^{k,p}(\Omega)$ a Banach space for any $k \in \N$. Note that $W^{0,p}(\Omega) = L^p(\Omega)$. 
\end{defn}

\section{(Non)closedness in Sobolev Spaces} \label{sec:nonclosedness}

In \citep{PRV20} it is shown that $\mathcal{RNN}_\rho^{[-B,B]^d}(d,N_1,\dots,N_{L-1},1)$ is \textit{not} closed in $L^p([-B,B]^d)$ for any $p \in (0,\infty)$, under mild assumptions satisfied by most commonly used activation functions (including ReLU, the rectified linear unit). Moreover, these sets of realized neural networks of a fixed architecture are \textit{not} closed in $C([-B,B]^d)$ with respect to the $L^\infty$ norm for most commonly used activation functions. However, sets of ReLU-realizations of two-layer networks \textit{are} closed in $C([-B,B]^d)$. These results are shown for $[-B,B]^d$, but generalize to any compact set $\Omega \subset \R^d$ with non-empty interior. 

In this work, we investigate the closedness of sets of realized neural networks in Sobolev spaces. Since convergence in Sobolev norm is stronger than $L^p$ convergence, Petersen, Raslan, and Voigtlaender anticipate that $\mathcal{RNN}_\rho^{[-B,B]^d}(d,N_1,\dots,N_{L-1},1)$ may be closed in Sobolev spaces \citep{PRV19}. We prove that this is often not the case. Provided that the activation function $\rho$ is $m$-times differentiable with bounded derivatives, these sets are not closed in $W^{m-1,p}([-B,B]^d)$ for any $p \in [1,\infty]$. 

\begin{theorem} \label{thm:m-(m+1)} 
	Let $m, d \in \N$, $p \in [1,\infty]$, and $B>0$. Define $\Omega = [-B,B]^d$. Consider a network architecture $(d,N_1,\dots,N_{L-1},1)$ with $L \geq 2$ and $N_{L-1} \geq 2$. Suppose that $\rho \in C^m(\R) \setminus C^{m+1}(\R)$ and all derivatives of $\rho$ up to order $m$ are locally $p$-integrable and bounded on compact sets. Then:
	\begin{itemize}
	    \item The set $\mathcal{RNN}_\rho^{\Omega}(d,N_1,\dots,N_{L-1},1)$ is not closed in $W^{m-1,p}(\Omega)$.
	    
	    \item If additionally $\rho^{(m)}$ is absolutely continuous and the weak derivative $\rho^{(m+1)}$ exists and is in $L^p(\Omega)$, then $\mathcal{RNN}_\rho^{\Omega}(d,N_1,\dots,N_{L-1},1)$ is not closed in $W^{m,p}(\Omega)$.
	\end{itemize}
\end{theorem}

\begin{proof}
	See Appendix \ref{subapp:m-(m+1)}. Similar to \citep{PRV20}, we construct a sequence of networks whose $\rho$-realizations converge to a target function that is not a $\rho$-realization of some network. In particular, we show order-$(m-1)$ (or order-$m$) Sobolev convergence to a target function that is $(m-1)$-times but not $m$-times differentiable, while \citep{PRV20} show $L^p$ convergence to a discontinuous step function.
\end{proof}


Section \ref{subsec:activation} lists several commonly used activation functions and whether they satisfy the assumptions of Theorem \ref{thm:m-(m+1)} for some value of $m$. 

Of course, convergence in order-$(m-1)$ Sobolev norm is stronger than convergence in lower-order Sobolev norm, so $\mathcal{RNN}_\rho^{\Omega}(d,N_1,\dots,N_{L-1},1)$ is not closed in lower-order Sobolev spaces either.

\begin{cor} \label{cor:lower-order}
	Let $d \in \N$ and $p \in [1,\infty]$. Suppose $\rho \in C^m(\R) \setminus C^{m+1}(\R)$ with bounded derivatives up to order $m$. Then $\mathcal{RNN}_\rho^{\Omega}(d,N_1,\dots,N_{L-1},1)$ is \textit{not} closed in $W^{k,p}(\Omega)$ for any $k \in \{0,\dots,m-1\}$, where $\Omega = [-B,B]^d$.
\end{cor}

\begin{proof}
	Note that convergence in $W^{m-1,p}(\Omega)$ implies convergence in $W^{k,p}(\Omega)$ for all $k \in \{0,\dots,m-1\}$. So we still have $f_n \to f$ in $W^{k,p}(\Omega)$ in the proof of Theorem \ref{thm:m-(m+1)}, but $f \notin \mathcal{RNN}_\rho^{\Omega}(d,N_1,\dots,N_{L-1},1)$.
\end{proof}

In Theorem \ref{thm:m-(m+1)}, we show that $\mathcal{RNN}_\rho^{\Omega}(d,N_1,\dots,N_{L-1},1)$ is not closed in order-$(m-1)$ Sobolev spaces for $\rho \in C^m(\R)$. For an analytic, bounded, and non-constant activation function $\rho$, we extend this result and prove that $\mathcal{RNN}_\rho^{\Omega}(d,N_1,\dots,N_{L-1},1)$ is not closed in any order Sobolev spaces. 

\begin{theorem} \label{thm:rho_smooth}
	Let $d \in \N$, $p \in [1,\infty]$, and $B>0$. Suppose that $\rho: \R \to \R$ is real analytic, bounded, and not constant, and that all derivatives $\rho^{(n)}$ of $\rho$ are bounded. Then for all possible neural network architectures $(d,N_1,\dots,N_{L-1},1)$ with $L \geq 2$ and $N_{L-1} \geq 2$ and all $k \in \N$, the set $\mathcal{RNN}_\rho^{[-B,B]^d}(d,N_1,\dots,N_{L-1},1)$ is \textit{not} closed in $W^{k,p}([-B,B]^d)$. 
\end{theorem}


\begin{proof}
	See Appendix \ref{subapp:rho_smooth}. Using arguments similar to those in \citep{PRV20}, we construct a sequence of networks whose $\rho$-realizations converge in any order Sobolev norm to an unbounded function, which cannot be a $\rho$-realization of some network since $\rho$ is bounded. Lemma \ref{lem:proj} in the proof is interesting in its own right, as it states that realized neural networks with analytic activation functions can approximate the coordinate projection maps to arbitrary accuracy in Sobolev norm. 
\end{proof}

Theorems \ref{thm:m-(m+1)} and \ref{thm:rho_smooth} show that $\mathcal{RNN}_\rho^{[-B,B]^d}(d,N_1,\dots,N_{L-1},1)$ is not closed in Sobolev spaces, although the smoothness of the activation function dictates the order of the Sobolev space in question. In particular, for $\rho \in C^m(\R) \setminus C^{m+1}(\R)$, Theorem \ref{thm:m-(m+1)} shows nonclosedness in order-$(m-1)$ or order-$m$ Sobolev spaces, while Theorem \ref{thm:rho_smooth} gives nonclosedness in all orders of Sobolev spaces for analytic activation functions. 
Both results apply to the $p=\infty$ case, indicating that sets of realized neural networks are not closed with respect to uniform convergence of a network and its derivatives.

\subsection{Closedness of Sets of Networks with Bounded Weights} \label{subsec:closedness}


The nonclosedness of sets of realized neural networks is undesirable if we only want to learn network target functions or if we want to prevent unbounded growth of network parameters. If we desire closedness, then we must modify the set of neural networks under consideration in some way. However, requiring closedness will necessarily constrain the set of target functions that we can approximate. 

Modifications to enforce closedness of sets of realized neural networks may include relaxing some assumptions on the activation function or placing restrictions on the network parameters. In this section, we discuss the closedness of sets of realized neural networks whose parameters are all bounded by the same constant. We define a norm on $\mathcal{NN}(d,N_1,\dots,N_L)$ and sets of realized neural networks with bounded norm.

\begin{defn}
	\citep{PRV20} Let $C>0$. Define
	\begin{equation*}
		\mathcal{NN}^C(d,N_1,\dots,N_L) = \{ \Phi \in \mathcal{NN}(d,N_1,\dots,N_L) : \|\Phi\|_{total} \leq C \},
	\end{equation*}
	as a set of neural networks with \textbf{uniformly bounded} weights, where
	\begin{equation*}
		\|\Phi\|_{total} = \max_{\ell=1,\dots,L} \|A_\ell\|_{max} + \max_{\ell=1,\dots,L} \|b_\ell\|_{max} 
	\end{equation*}
	and $\|\cdot\|_{max}$ equals the absolute value of the entry of largest magnitude from a matrix or vector. For $\Omega \subset \R^d$ and $\rho : \R \to \R$, also define $$\mathcal{RNN}_\rho^{\Omega,C}(d,N_1,\dots,N_L) \defeq R_\rho^\Omega\big(\mathcal{NN}^C(d,N_1,\dots,N_L)\big)$$
	as a set of realized neural networks with uniformly bounded weights and biases. 
\end{defn}

Petersen, Raslan, and Voigtlaender show that $\mathcal{RNN}_\rho^{[-B,B]^d}(d,N_1,\dots,N_{L-1},1)$ is not closed in $L^p([-B,B]^d)$ or in $C([-B,B]^d)$ with respect to the $L^\infty$ norm. However, sets of realized neural networks with uniformly bounded parameters \textit{are} closed (in fact, compact) in these spaces by Proposition 3.5 in \citep{PRV20}.

Since Sobolev convergence is stronger than $L^p$ convergence, $\mathcal{RNN}_\rho^{\Omega,C}(d,N_1,\dots,N_L)$ is also closed in Sobolev spaces.

\begin{cor} \label{cor:bounded_weights}
	Let $\Omega \subset \R^d$ be compact, $C>0$, $p \in [1,\infty]$, $k \in \N$, and $\rho: \R \to \R$ be continuous. Then $\mathcal{RNN}_\rho^{\Omega,C}(d,N_1,\dots,N_L)$ is closed in $W^{k,p}(\Omega)$.
\end{cor}

\begin{proof}
	If $(f_n)_{n \in \N} \subset \mathcal{RNN}_\rho^{\Omega,C}(d,N_1,\dots,N_L)$ satisfies $\|f_n-f\|_{W^{k,p}(\Omega)} \to 0$ for some $f$, then $f_n \to f$ in $L^p$ norm. Thus, $f \in \mathcal{RNN}_\rho^{\Omega,C}(d,N_1,\dots,N_L)$ because this set is closed in $L^p(\Omega)$ by Proposition 3.5 in \citep{PRV20}.
\end{proof}

The nonclosedness of $\mathcal{RNN}_\rho^{\Omega}$ in Sobolev spaces has significant consequences for approximating functions using neural networks. Indeed, it says that for any architecture $S=(d,N_1,\dots,N_{L-1},1)$ with $L \geq 2$ and $N_{L-1} \geq 2$, there is a non-network target function $f \in \overline{\mathcal{RNN}_\rho^{\Omega}(S)} \setminus \mathcal{RNN}_\rho^{\Omega}(S)$, where the closure can be taken with respect to Sobolev norm of the appropriate order. Combined with the closedness of the set $\mathcal{RNN}_\rho^{\Omega,C}(S)$ with uniformly bounded weights, this means that if $\|R_\rho(\Omega)(\Phi_n) - f\|_{W^{k,p}(\Omega)} \to 0$ for some sequence of networks $(\Phi_n)_{n \in \N}$ with architecture $S$, then $\|\Phi_n\|_{total} \to \infty$. This may explain the phenomenon of weights growing without bound that sometimes occurs when training neural networks, and it indicates that using Sobolev training for neural networks can still lead to such a growth of parameters for some target functions.

\subsection{Network Parameter Growth Rates} \label{subsec:rates}

We have seen that convergence of a neural network to a non-network target function requires infinite growth of at least one network weight, but we would like to know how quickly the weights must grow relative to the decreasing approximation error. An interesting area of further research would be to describe this relationship for fixed architectures in full generality based on characteristics of the activation function, target function, and network architecture. We describe the approximation error decay for most activation functions in the case of a certain sequence of networks learning the derivative of the activation function.

\begin{prop} \label{prop:rates}
    Let $\rho \in C^m(\R)$ for some $m \geq 2$, the softsign function, or ELU (cf.\ Table \ref{table:activation}). Let $(h_n)_{n=1}^\infty = (R_\rho^\R(\Phi_n))_{n=1}^\infty$ be the sequence of realized neural networks in $\mathcal{RNN}(1,2,1)$ from the proof of Theorem \ref{thm:m-(m+1)}. Further let $p \in [1,\infty]$ and let $\Omega \subset \R$ be a compact, measurable set with nonempty interior. Then $\|h_n - \rho'\|_{L^p(\Omega)} \leq C_p/\|\Phi_n\|_{total}$ for a constant $C_p$ depending on $p$ but not $n$.
\end{prop}

\begin{proof}
    See Appendix \ref{subapp:rates}.
\end{proof}

In other words, Proposition \ref{prop:rates} shows that the $L^p$ approximation error is approximately inversely proportional to the networks' total norms for a sequence of networks learning the activation function derivative. Note that when the activation function is not analytic, its derivative is a non-network target function.

\subsection{Commonly Used Activation Functions} \label{subsec:activation}

Though some of the assumptions on the activation function required by Theorems \ref{thm:m-(m+1)} and \ref{thm:rho_smooth} seem strong, they are satisfied by many commonly used activation functions. Thus, the results of these theorems apply, and sets of neural networks are not closed in various orders of Sobolev spaces for $\Omega$ compact and $p \geq 1$. Table \ref{table:activation} summarizes which results apply to several common activation functions.


\begin{longtable}{|l|l|l|l|}  \hline
		\multirow{2}{*}{\textbf{Name}} & \multirow{2}{*}{$\rho(x)$} & \textbf{Smoothness/} & $\mathcal{RNN}_\rho^\Omega \Big.$ \\
		 & & \textbf{Boundedness} & \textbf{not closed in} \\ \hline \endfirsthead
		
		
		
		  
		Rectified Linear & \multirow{2}{*}{$\max\{0,x\}$} & $\Big. C(\R)$, abs. cont., & \multirow{2}{*}{$W^{0,p}(\Omega)$ [1]} \\ 
		Unit (ReLU) & & $\rho' \in L^p(\Omega)$ & \\ \hline 
		Exponential Linear & \multirow{2}{*}{$x \cdot \chi_{x \geq 0} + (e^x-1) \cdot \chi_{x < 0}$} & $\Big. C^1(\R)$, $\rho'$ abs. cont., & \multirow{2}{*}{$W^{1,p}(\Omega)$}  \\
		Unit (ELU) & & $\rho'' \in L^p(\Omega)$ & \\ \hline 
		\multirow{2}{*}{Softsign} & \multirow{2}{*}{$\frac{x}{1+|x|}$} & $\Big. C^1(\R)$, $\rho'$ abs. cont., & \multirow{2}{*}{$W^{1,p}(\Omega)$} \\
		 & & $\rho'' \in L^p(\Omega)$ & \\ \hline 
		Inverse Square Root & \multirow{2}{*}{$x \cdot \chi_{x \geq 0} + \frac{x}{\sqrt{1+ax^2}} \cdot \chi_{x < 0}$} & $\Big. C^2(\R)$, $\rho''$ abs. cont., & \multirow{2}{*}{$W^{2,p}(\Omega)$}  \\
		Linear Unit ($a > 0$) & & $\rho''' \in L^p(\Omega)$ & \\ \hline
		Inverse Square Root & \multirow{2}{*}{$\frac{x}{\sqrt{1+ax^2}}$} & real analytic, all $\Big.$ & \multirow{2}{*}{$W^{k,p}(\Omega)$ for all $k$}  \\
		Unit ($a > 0$) & &  derivatives bounded & \\ \hline
		\multirow{2}{*}{Sigmoid} & \multirow{2}{*}{$\frac{1}{1+e^{-x}}$} & real analytic, all $\Big.$ & \multirow{2}{*}{$W^{k,p}(\Omega)$ for all $k$}  \\
		 & &  derivatives bounded [2] &\\ \hline
		\multirow{2}{*}{tanh} & \multirow{2}{*}{$\frac{e^x - e^{-x}}{e^x + e^{-x}}$} & real analytic, all $\Big.$ & \multirow{2}{*}{$W^{k,p}(\Omega)$ for all $k$}  \\
		 & & derivatives bounded [2] &\\ \hline
		\multirow{2}{*}{arctan} & \multirow{2}{*}{$\arctan(x)$} & real analytic, all  $\Big.$ & \multirow{2}{*}{$W^{k,p}(\Omega)$ for all $k$}  \\
		 & & derivatives bounded [3] &\\ \hline
		 
		\captionsetup{width=0.85\linewidth}
		\caption{Many activation functions used in practice satisfy some smoothness and boundedness properties so that Theorems \ref{thm:m-(m+1)} and \ref{thm:rho_smooth} apply. Thus, $\mathcal{RNN}_\rho^\Omega$ is not closed in various orders of Sobolev spaces for $\Omega$ compact and $p \in [1,\infty]$. Some results in the table are found in the following references: [1] \citep{PRV20} [2] \citep{MW93} [3] \citep{AL10}.} \label{table:activation} 
\end{longtable}

The smoothness and boundedness assumptions for the ReLU, exponential linear unit, softsign, and inverse square root linear unit can be checked by hand. The real analyticity of the other activation functions is established by properties from \citep{KP02}.

For these activation functions, sets of realized neural networks are not closed in Sobolev spaces of certain orders. Thus, using Sobolev training still allows these networks to learn non-network target functions, although doing so will cause unbounded growth of parameters, which follows from Corollary \ref{cor:bounded_weights}.

\section{Experimental Results} \label{sec:experiments}


We now show some experimental results that demonstrate the nonclosedness of sets of realized neural networks in Sobolev spaces and examine the rate of parameter growth. Specifically, we use Sobolev training with an Adam optimizer \citep{KB17} to produce sequences of neural networks that approximate non-network target functions. 

For activation functions that are $m$-times but not $(m+1)$-times differentiable, we know from the proof of Theorem \ref{thm:m-(m+1)} that there is a sequence of networks that converges in Sobolev norm to the derivative of the activation function. The derivative is not a realized neural network because it is only $(m-1)$-times differentiable. With this as motivation, each trial of our experiment trains a network to learn a randomly generated $(m-1)$-times differentiable target function in Sobolev norm. We consistently observe a rapidly decreasing approximation error and a fairly steady growth of network weights.

\begin{figure}[ht]
	\centering
	\begin{subfigure}[b]{0.32\linewidth}
	    \centering
	    \includegraphics[scale=0.33]{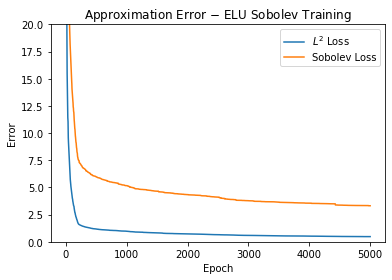}
	    \caption{Training error. \label{subcap:ELU_Loss_S}}
	\end{subfigure}
	\begin{subfigure}[b]{0.32\linewidth}
	    \centering
	    \includegraphics[scale=0.33]{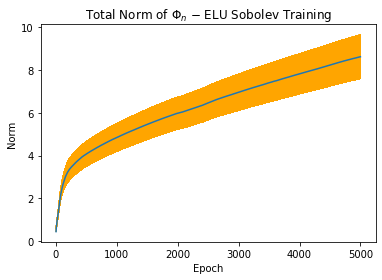}
	    \caption{The network norm. \label{subcap:ELU_Norm_S}}
	\end{subfigure}
	\begin{subfigure}[b]{0.32\linewidth}
	    \centering
	    \includegraphics[scale=0.33]{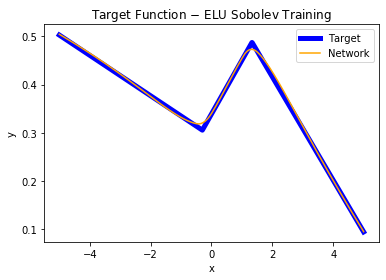}
	    \caption{The target function. \label{subcap:ELU_Target_S}}
	\end{subfigure}

	\captionsetup{width=0.85\linewidth}
	\caption{For each of 100 target functions $f$, we train an ELU network $\Phi$ in $\mathcal{NN}(1,10,1)$ to minimize $\|R_\rho^{[-5,5]}(\Phi)-f\|_{W^{1,2}}$. (\subref{subcap:ELU_Loss_S}) The best $W^{1,2}$ training loss achieved thus far is plotted at each epoch and averaged over all 100 experiments. (\subref{subcap:ELU_Norm_S}) The total network norm is averaged over all 100 experiments with 95\% confidence bands. (\subref{subcap:ELU_Target_S}) An example target function is plotted along with the realized neural network after training.}\label{cap:ELU_S}
\end{figure}

Figure \ref{cap:ELU_S} shows the results of 100 trails of ELU networks learning non-network target functions in Sobolev norm. Since ELU is $C^1$ but not $C^2$, we train the networks to learn randomly generated piecewise linear ($C^0$ but not $C^1$) functions in order-1 Sobolev space. We see that the $L^2$ and Sobolev approximation errors decrease rather quickly, while the total network norm increases quickly at first and then at a fairly steady rate.

These results are consistent with Theorem \ref{thm:m-(m+1)}. We see that networks are able to closely approximate non-network target functions, which is evidence of the nonclosedness of sets of realized neural networks in Sobolev spaces. Appendix \ref{app:exp} provides similar results for the ISRLU activation function. Since ISRLU is $C^2$ and not $C^3$, we train these networks to learn randomly generated piecewise quadratic target functions in order-2 Sobolev norm.

For real analytic activation functions, Lemma \ref{lem:proj} indicates that realized neural networks can approximate coordinate projection maps to arbitrary accuracy. We trained a sigmoid neural network to learn the projection map $P_1: \R^2 \to \R$ given by $P_1(x_1,x_2)=x_1$ in order-2 Sobolev norm. We reset the network weights and repeat the training process 100 times to assess whether sigmoid networks consistently learn $P_1$, which is a non-network function.

\begin{figure}[ht]
	\centering
	\begin{subfigure}[b]{0.32\linewidth}
	    \centering
	    \includegraphics[scale=0.33]{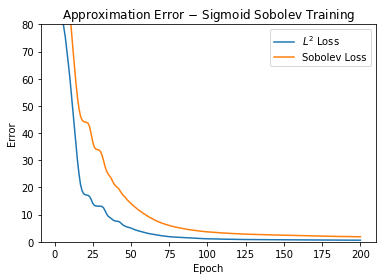}
	    \caption{Training error. \label{subcap:Sig_Loss_S}}
	\end{subfigure}
	\begin{subfigure}[b]{0.32\linewidth}
	    \centering
	    \includegraphics[scale=0.33]{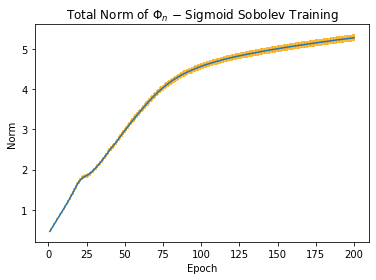}
	    \caption{The network norm. \label{subcap:Sig_Norm_S}}
	\end{subfigure}
	\begin{subfigure}[b]{0.32\linewidth}
	    \centering
	    \includegraphics[scale=0.42]{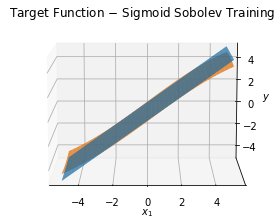}
	    \caption{The target function. \label{subcap:Sig_Target_S}}
	\end{subfigure}

	\captionsetup{width=0.85\linewidth}
	\caption{In 100 repetitions, we train a sigmoid network $\Phi$ in $\mathcal{NN}(1,10,1)$ to minimize $\|R_\rho^{[-5,5]^2}(\Phi)-P_1\|_{W^{2,2}}$. (\subref{subcap:Sig_Loss_S}) The best $W^{2,2}$ training loss achieved thus far is plotted at each epoch and averaged over all 100 experiments. (\subref{subcap:Sig_Norm_S}) The total network norm is averaged over all 100 experiments with 95\% confidence bands. (\subref{subcap:Sig_Target_S}) The target function (blue) is plotted along with the realized neural network (orange) after training.}\label{cap:Sig_S}
\end{figure}

Figure \ref{cap:Sig_S} shows the results of 100 trials of sigmoid networks learning $P_1$ in order-2 Sobolev norm. Note that $P_1$ is a non-network target function because it is unbounded, as discussed in the proof of Theorem \ref{thm:rho_smooth}. As in Figure \ref{cap:ELU_S}, we see that the $L^2$ and Sobolev approximation errors decrease rather quickly, while the total network norm increases at a fairly steady rate.

These results are consistent with Theorem \ref{thm:rho_smooth}. We see that networks with analytic activation functions are able to closely approximate non-network target functions, which is further evidence of the nonclosedness of sets of realized neural networks in Sobolev spaces. Our experiments were done with shallow networks, but we expect the same nonclosedness to be demonstrated with deep networks since they are necessarily more expressive.

\subsection{Network Parameter Growth Rates} \label{subsec:rates_exp}

The results in Figures \ref{cap:ELU_S} and \ref{cap:Sig_S} show that realized neural networks are able to approximate non-network target functions in Sobolev norm with growing parameters, but the relationship between the approximation error and parameter growth is not clear. In this section, we provide numerical evidence for one case of Proposition \ref{prop:rates} by showing that $L^2$ approximation error and the network's total norm are approximately inversely proportional for a softsign network learning the softsign derivative. We analyze the relationship by making a scatter plot of approximation error against the total norm. We do the same for the Sobolev error, even though we do not have a theoretical result for this case.

\begin{figure}[ht]
	\centering
	\begin{subfigure}[b]{0.49\linewidth}
	    \centering
	    \includegraphics[width=0.9\linewidth]{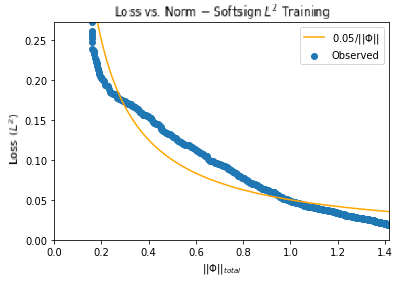}
	    \caption{$L^2$ error vs.\ network norm. \label{subcap:rates_N}}
	\end{subfigure}
	\begin{subfigure}[b]{0.49\linewidth}
	    \centering
	    \includegraphics[width=0.89\linewidth]{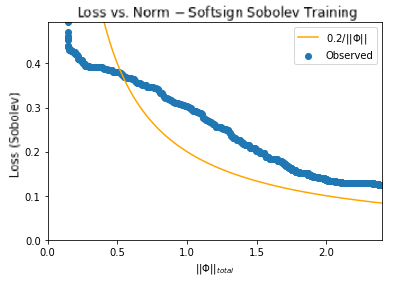}
	    \caption{Sobolev error vs.\ network norm. \label{subcap:rates_S}}
	\end{subfigure}

	\captionsetup{width=0.85\linewidth}
	\caption{We train a softsign network $\Phi$ in $\mathcal{NN}(1,2,1)$ to learn the softsign derivative and scatter approximation error against the total norm of $\Phi$. (\subref{subcap:rates_N}) Training error vs.\ network norm using $L^2$ training. (\subref{subcap:rates_S}) Training error vs.\ network norm using order-1 Sobolev training.}\label{cap:rates}
\end{figure}

Figure \ref{cap:rates}(\subref{subcap:rates_N}) shows that the relationship between $L^2$ approximation error and the total norm is approximately inversely proportional, as stated in Proposition \ref{prop:rates}. Note that our networks are trained using an Adam optimizer \citep{KB17} and thus do not follow the sequence $(h_n)_{n=1}^\infty$ from Proposition \ref{prop:rates}, but we still observe the expected relationship between approximation error and weights. The relationship for Sobolev error in Figure \ref{cap:rates}(\subref{subcap:rates_S}) is not as clear. An interesting question of further research is how the relationship between error and norm generalizes to other target functions, sequences of networks, network architectures, and higher-order Sobolev norms.

\subsection*{Declaration of Competing Interest}

The authors declare that they have no known competing financial interests or personal relationships that could have appeared to influence the work reported in this paper.

\subsection*{Acknowledgements}

This material is based upon work supported by the National Science Foundation Graduate Research Fellowship Program under Grant No.\ DGE-1650112. Any opinions, findings, and conclusions or recommendations expressed in this material are those of the author(s) and do not necessarily reflect the views of the National Science Foundation.

SM is funded by grant NSF DGE GRFP \#1650112.
AC is funded by grants NSF DMS \#1819222, \#2012266, and Russell Sage Foundation grant 2196.


\bibliography{NoclosedNN}{}

\begin{thebibliography}{}

\bibitem[Adegoke and Layeni, 2010]{AL10}
Adegoke, K. and Layeni, O. (2010).
\newblock The higher derivatives of the inverse tangent function and rapidly
  convergent {BBP}-type formulas for pi.
\newblock {\em Applied Mathematics E-Notes}, 10:70--75.

\bibitem[Cybenko, 1989]{C89}
Cybenko, G. (1989).
\newblock Approximation by superpositions of a sigmoidal function.
\newblock {\em Mathematics of Control, Signals, and Systems}, 2(4):303--314.

\bibitem[Czarnecki et~al., 2017]{COJSP17}
Czarnecki, W.~M., Osindero, S., Jaderberg, M., Swirszcz, G., and Pascanu, R.
  (2017).
\newblock Sobolev training for neural networks.
\newblock In {\em 31st Conference on Neural Information Processing Systems
  (NIPS 2017)}, pages 1--10.

\bibitem[Folland, 1999]{F99}
Folland, G. (1999).
\newblock {\em Real Analysis}.
\newblock Pure and Applied Mathematics. John Wiley \& Sons, New York, 2
  edition.

\bibitem[Girosi and Poggio, 1990]{GP90}
Girosi, F. and Poggio, T. (1990).
\newblock Networks and the best approximation property.
\newblock {\em Biological Cybernetics}, 63:169--176.

\bibitem[G\"{u}hring et~al., 2020]{GRK20}
G\"{u}hring, I., Raslan, M., and Kutyniok, G. (2020).
\newblock Expressivity of deep neural networks.
\newblock arXiv:2007.04759v1.

\bibitem[Hornik, 1991]{H91}
Hornik, K. (1991).
\newblock Approximation capabilities of multilayer feedforward networks.
\newblock {\em Neural Networks}, 4(2):251--257.

\bibitem[Hornik et~al., 1989]{HSW89}
Hornik, K., Stinchcombe, M., and White, H. (1989).
\newblock Multilayer feedforward networks are universal approximators.
\newblock {\em Neural Networks}, 2(5):359--366.

\bibitem[Hornik et~al., 1990]{HSW90}
Hornik, K., Stinchcombe, M., and White, H. (1990).
\newblock Universal approximation of an unknown mapping and its derivatives
  using multilayer feedforward networks.
\newblock {\em Neural Networks}, 3:551--560.

\bibitem[Kainen et~al., 1999]{KKV99}
Kainen, P.~C., K\r{u}rkov\'{a}, V., and Vogt, A. (1999).
\newblock Approximation by neural networks is not continuous.
\newblock {\em Neurocomputing}, 29:47--56.

\bibitem[Kainen et~al., 2000]{KKV00}
Kainen, P.~C., K\r{u}rkov\'{a}, V., and Vogt, A. (2000).
\newblock Best approximation by {Heaviside} perceptron networks.
\newblock {\em Neural Networks}, 13:695--697.

\bibitem[Kainen et~al., 2020]{KKV20}
Kainen, P.~C., K\r{u}rkov\'{a}, V., and Vogt, A. (2020).
\newblock Approximative compactness of linear combinations of characteristic
  functions.
\newblock {\em Journal of Approximation Theory}, 257:1--17.

\bibitem[Kingma and Ba, 2017]{KB17}
Kingma, D.~P. and Ba, J.~L. (2017).
\newblock Adam: A method for stochastic optimization.
\newblock arXiv:1412.6980v9.

\bibitem[Krantz and Parks, 2002]{KP02}
Krantz, S.~G. and Parks, H.~R. (2002).
\newblock {\em A Primer of Real Analytic Functions}.
\newblock Birkh\"{a}user Advanced Texts. Birkh\"{a}user, Boston, MA, 2 edition.

\bibitem[Krizhevsky et~al., 2012]{KSH12}
Krizhevsky, A., Sutskever, I., and Hinton, G. (2012).
\newblock Imagenet classification with deep convolutional neural networks.
\newblock {\em Advances in Neural Information Processing Systems},
  25(2):1097--1105.

\bibitem[K\r{u}rkov\'{a}, 1995]{K95}
K\r{u}rkov\'{a}, V. (1995).
\newblock Approximation of functions by perceptron networks with bounded number
  of hidden units.
\newblock {\em Neural Networks}, 8(5):745--750.

\bibitem[Mhaskar, 1997]{M97}
Mhaskar, H.~N. (1997).
\newblock On smooth activation functions.
\newblock In {\em Mathematics of Neural Networks}, pages 275--279. Springer.

\bibitem[Minai and Williams, 1993]{MW93}
Minai, A.~A. and Williams, R.~D. (1993).
\newblock On the derivatives of the sigmoid.
\newblock {\em Neural Networks}, 6(6):845--853.

\bibitem[Petersen et~al., 2019]{PRV19}
Petersen, P., Raslan, M., and Voigtlaender, F. (2019).
\newblock Unfavorable structural properties of the set of neural networks with
  fixed architecture.
\newblock In {\em 2019 13th International conference on Sampling Theory and
  Applications (SampTA)}, pages 1--4.

\bibitem[Petersen et~al., 2020]{PRV20}
Petersen, P., Raslan, M., and Voigtlaender, F. (2020).
\newblock Topological properties of the set of functions generated by neural
  networks of fixed size.
\newblock {\em Foundations of Computational Mathematics}, pages 1--70.

\end{thebibliography}
\bibliographystyle{apalike}

\appendix

\section{Proofs of Results from Section \ref{sec:nonclosedness}} \label{app:proofs}

\subsection{Proof of Theorem \ref{thm:m-(m+1)}} \label{subapp:m-(m+1)}

The proof of Theorem \ref{thm:m-(m+1)} uses the following lemma.

\begin{lem} \label{lem:J}
	Let $d \in \N$, $L \geq 2$, $B>0$, and $D>0$. Set $\Omega = [-B,B]^d$. Suppose $\rho \in C^m(\R) \setminus C^{m+1}(\R)$ for some $m \in \N$. Then there exists $\Phi \in \mathcal{NN}(d,1,\dots,1)$ with $L-1$ layers such that $R_\rho^\Omega(\Phi)(\Omega) \supset [-D,D]$.
\end{lem}

\begin{proof}[Proof of Lemma \ref{lem:J}]
	Since $\rho \in C^m(\R) \setminus C^{m+1}(\R)$ for some $m \in \N$, $\rho$ is not a constant function. Moreover, since $\rho \in C^1(\R)$, there exists an open interval $U \subset \R$ such that $0 \notin \rho'(U)$. Hence, $\rho(U)$ is not a single point and must be an interval with non-empty interior. Next, choose $A_1 \in \R^{1 \times d}$ such that $$A_1 \Omega \supset U.$$ Since $U$ has non-empty interior, we can iteratively choose $A_j \in \R^{1 \times 1}$ such that
	\begin{equation*}
		A_j \rho (A_{j-1} \rho(\cdots \rho(A_1 \Omega) \cdots ) ) \supset U
	\end{equation*}
	for $j=2,\dots,L-2$. Again since $U$ has non-empty interior, we can choose $A_{L-1} \in \R^{1 \times 1}$ and $b_{L-1} \in \R$ such that
	\begin{equation*}
		A_{L-1} \rho (A_{L-2} \rho(\cdots \rho(A_1 \Omega) \cdots ) ) + b_{L-1} \supset [-D,D].
	\end{equation*}
	Finally, if we let
	\begin{equation*}
		\Phi = \big( (A_1,0),\dots,(A_{L-2},0),(A_{L-1},b_{L-1}) \big)
	\end{equation*}
	then $R_\rho^\Omega(\Phi)(\Omega) \supset [-D,D]$ by construction.
\end{proof}

We will also use a lemma about the existence of strong $L^p$ derivatives on bounded sets, which is similar to Exercise 8.9 in \citep{F99}.

\begin{lem} \label{lem:Folland}
    Let $p \in [1,\infty]$ and let $K \subset \R$ be compact with $f \in L^p(K)$. If $f$ is absolutely continuous on every bounded interval and its pointwise derivative $f'$ is in $L^p(K)$, then
    \begin{equation*}
        \lim_{y \to 0} \left\| \frac{f(\cdot+y) - f(\cdot)}{y} - f'(\cdot) \right\|_{L^p(K)} = 0.
    \end{equation*}
    In this case, we call $f'$ the strong $L^p$ derivative of $f$ on $K$.
\end{lem}

\begin{proof}[Proof of Lemma \ref{lem:Folland}]
    Let $K \subset \R$ be compact with $f \in L^p(K)$. Suppose $f$ is absolutely continuous on every bounded interval and its pointwise derivative $f'$ is in $L^p(K)$. Since $f$ is absolutely continuous on every bounded interval, we can write
    \begin{align*}
        \frac{f(x+y) - f(x)}{y} - f'(x) &= \frac{1}{y} \int_x^{x+y} f'(t) dt - f'(x) \\
        &= \frac{1}{y} \int_0^y f'(x+t) dt - f'(x) \\
        &= \frac{1}{y} \int_0^y [f'(x+t)-f'(x)] dt. 
    \end{align*}
    For $p<\infty$, Minkowski's inequality for integrals gives
    \begin{align*}
        \left\| \frac{f(\cdot+y) - f(\cdot)}{y} - f'(\cdot) \right\|_{L^p(K)} &= \left[ \int_K \left( \frac{1}{y} \int_0^y [f'(x+t)-f'(x)] dt \right)^p dx \right]^{1/p} \\
        &\leq \frac{1}{y} \int_0^y \left( \int_K [f'(x+t)-f'(x)]^p dx \right)^{1/p} dt \\
        &= \frac{1}{y} \int_0^y \left\| f'(\cdot+t)-f'(\cdot) \right\|_{L^p(K)} dt.
    \end{align*}
    For $p=\infty$, we can similarly write
    \begin{align*}
        \left\| \frac{f(\cdot+y) - f(\cdot)}{y} - f'(\cdot) \right\|_{L^\infty(K)} &= \sup_{x \in K} \left| \frac{1}{y} \int_0^y [f'(x+t)-f'(x)] dt \right| \\
        &\leq \frac{1}{y} \int_0^y \sup_{x \in K} |f'(x+t)-f'(x)| dt \\
        &= \frac{1}{y} \int_0^y \left\| f'(\cdot+t)-f'(\cdot) \right\|_{L^\infty(K)} dt.
    \end{align*}
    (We do not have to consider the essential supremum over $K$ because $f$ is absolutely continuous on $K$. Moreover, if $y<0$ then the negative sign can be absorbed by switching the limits of integration, which still allows us to move the supremum inside without flipping the inequality.) Now let $\epsilon>0$. Since $f' \in L^p(K)$, there exists $\delta>0$ such that
    \begin{equation*}
        \left\| f'(\cdot+t)-f'(\cdot) \right\|_{L^p(K)} \leq \epsilon
    \end{equation*}
    whenever $|t|<\delta$. If $|y|<\delta$ so we are only integrating over $t$ with $|t|<\delta$, we have
    \begin{equation*}
        \left\| \frac{f(\cdot+y) - f(\cdot)}{y} - f'(\cdot) \right\|_{L^p(K)} \leq \frac{1}{y} \int_0^y \epsilon \, dt = \epsilon.
    \end{equation*}
    Thus, $\lim_{y \to 0} \left\| y^{-1}\big(f(\cdot+y) - f(\cdot)\big) - f'(\cdot) \right\|_{L^p(K)} = 0$ as desired.
\end{proof}

\begin{proof}[Proof of Theorem \ref{thm:m-(m+1)}]
Given $p \in [1,\infty]$ and $\rho \in C^m(\R) \setminus C^{m+1}(\R)$ as in the statement of the theorem, we will construct a sequence of functions $(f_n)_{n=1}^\infty$ in $R_\rho^\Omega(d,1,\dots,1,2,1)$ whose $W^{m-1,p}(\Omega)$-limit $f$ is not in $C^m(\Omega)$. Since $\rho \in C^m(\R)$ implies
\begin{equation*}
	\mathcal{RNN}_\rho^{\Omega}(d,1,\dots,1,2,1) \subset C^m(\Omega),
\end{equation*}
we have $f \notin \mathcal{RNN}_\rho^{\Omega}(d,1,\dots,1,2,1)$ and hence $\mathcal{RNN}_\rho^{\Omega}(d,1,\dots,1,2,1)$ is not closed in $W^{m-1,p}(\Omega)$. Since $\tilde{N}_\ell \geq N_\ell$ for $\ell=1,\dots,L-1$ implies 
\begin{equation*}
	\mathcal{RNN}_\rho^{\Omega}(d,N_1,\dots,N_{L-1},N_L) \subset \mathcal{RNN}_\rho^{\Omega}(d,\tilde{N}_1,\dots,\tilde{N}_{L-1},N_L)
\end{equation*} 
by Lemma 2.5 in \citep{PRV20}, the nonclosedness of $\mathcal{RNN}_\rho^{\Omega}(d,N_1,\dots,N_{L-1},1)$ holds for any architecture $(d,N_1,\dots,N_{L-1},1)$ with $L \geq 2$ and $N_{L-1} \geq 2$. \\

To construct $(f_n)_{n=1}^\infty$, first note that $\rho \notin C^{m+1}(\R)$ implies $\rho \notin C^{m+1}([-C,C])$ for some $C>0$. Let $\Phi \in \mathcal{NN}(d,1,\dots,1)$ have $L-1$ layers such that $J(x) \defeq R_\rho^\Omega(\Phi)(x)$ satisfies $J(\Omega) \supset [-C,C]$ as in Lemma \ref{lem:J}. Define a sequence of neural networks $\Phi_n = \big( (A_1^n,b_1^n), (A_2^n,b_2^n) \big) \in \mathcal{NN}(1,2,1)$ by
\begin{equation*}
	A_1^n = \begin{pmatrix}
		1	\\
		1
	\end{pmatrix} \in \R^{2 \times 1},	\hspace{5mm} b_1^n = \begin{pmatrix}
		1/n	\\
		0
	\end{pmatrix} \in \R^2, \hspace{5mm}	A_2^n = \begin{pmatrix}
		n	&-n	
	\end{pmatrix} \in \R^{1 \times 2}, \hspace{5mm} b_2^n = 0 \in \R^1.
\end{equation*}
Next define $h_n : \R \to \R$ by
\begin{equation*}
	h_n(x) = R_\rho^\R(\Phi_n)(x) = n \rho(x + 1/n) - n \rho(x) = \frac{ \rho( x + 1/n ) - \rho(x) }{ 1/n }
\end{equation*}
and let $f_n = h_n \circ J = R_\rho^\Omega(\Phi_n \bullet \Phi)$ so that
\begin{equation*}
	f_n(x) = \frac{ \rho( J(x) + 1/n ) - \rho(J(x)) }{ 1/n }
\end{equation*}
for $n \in \N$. Notice that $f_n \in C^m(\R^d)$ for all $n \in \N$, and for any indices $i_1,\dots,i_l$ (not necessarily distinct) with $l \leq m$, we have
\begin{equation*}
	\frac{\partial^l}{\partial x_{i_1} \cdots \partial x_{i_l}} f_n(x) = \sum_{\pi \in \Pi} \left( \frac{ \rho^{(|\pi|)}( J(x) + 1/n ) - \rho^{(|\pi|)}(J(x)) }{1/n} \right)	\cdot \prod_{B \in \pi}  \frac{ \partial^{|B|} J(x) }{ \prod_{j \in B} \partial x_j }							
\end{equation*}
by Fa\`{a} di Bruno's formula (here the sum is taken over all set partitions $\pi$ of the set $\{1,\dots,l\}$ and the product is taken over all blocks $B$ in the partition $\pi$). We then have
\begin{equation}\label{eq:f_n^l}
	\lim_{n \to \infty} \frac{\partial^l}{\partial x_{i_1} \cdots \partial x_{i_l}} f_n(x) = \sum_{\pi \in \Pi} \left( \rho^{(|\pi|+1)}(J(x)) \right)	\cdot \prod_{B \in \pi} \frac{ \partial^{|B|} J(x) }{ \prod_{j \in B} \partial x_j } = \frac{\partial^l}{\partial x_{i_1} \cdots \partial x_{i_l}} \rho'(J(x))						
\end{equation}
pointwise for $l<m$. For $l=0,\dots,m-1$, $\rho^{(l+1)}$ is continuous and bounded on every bounded interval, so $\rho^{(l)}$ is Lipschitz continuous and therefore absolutely continuous on every bounded interval. 
Moreover, $\rho^{(l+1)} \in L^p_{\mathrm{loc}}(\R)$ by assumption. Hence, by 
Lemma, \ref{lem:Folland}, the pointwise derivative of $\rho^{(l)}$ agrees with its strong $L^p$ derivative on $J(\Omega)$. 
Thus, we have
\begin{equation*}
	\lim_{n \to \infty} \left\| \frac{\rho^{(l)}(J(\cdot) + 1/n)-\rho^{(l)}(J(\cdot))}{1/n} - \rho^{(l+1)}(J(\cdot)) \right\|_{L^p(\Omega)} = 0
\end{equation*}
for $l=1,\dots,m-1$. Since all terms in Equation (\ref{eq:f_n^l}) are bounded 
on $\Omega$, it follows that
\begin{equation*}
	\lim_{n \to \infty} \left\| f_n - \rho'(J(\cdot)) \right\|_{W^{m-1,p}(\Omega)} = \lim_{n \to \infty} \sum_{|\alpha| \leq m-1} \left\| D^\alpha \big(f_n-\rho'(J(\cdot))\big) \right\|_{L^p(\Omega)} = 0.
\end{equation*}
That is, the sequence $(f_n)_{n=1}^\infty$ converges in $W^{m-1,p}(\Omega)$ to the function $\Omega \to \R$ given by $x \mapsto f(x) \defeq \rho'(J(x))$. Since $[-C,C] \subset J(\Omega)$, we have $f \notin C^m(\Omega)$ and therefore $f \notin R_\rho^\Omega(d,1,\dots,1,2,1)$. Hence, $R_\rho^\Omega(d,1,\dots,1,2,1)$ is \textit{not} closed in $W^{m-1,p}(\Omega)$. \\

If in addition $\rho^{(m)}$ is absolutely continuous and the weak derivative $\rho^{(m+1)}$ exists and is in $L^p_{\mathrm{loc}}(\R)$, then 
Lemma \ref{lem:Folland} gives
\begin{equation*}
	\lim_{n \to \infty} \left\| \frac{\rho^{(m)}(J(\cdot) + 1/n)-\rho^{(m)}(J(\cdot))}{1/n} - \rho^{(m+1)}(J(\cdot)) \right\|_{L^p(\Omega)} = 0
\end{equation*}
and therefore
\begin{equation*}
	\lim_{n \to \infty} \left\| f_n - \rho'(J(\cdot)) \right\|_{W^{m,p}(\Omega)} = \lim_{n \to \infty} \sum_{|\alpha| \leq m} \left\| D^\alpha \big(f_n-\rho'(J(\cdot))\big) \right\|_{L^p(\Omega)} = 0.
\end{equation*}
Hence, $R_\rho^\Omega(d,1,\dots,1,2,1)$ is \textit{not} closed in $W^{m,p}(\Omega)$ under this further assumption on $\rho$.
\end{proof}

\subsection{Proof of Theorem \ref{thm:rho_smooth}} \label{subapp:rho_smooth}

The proof of Theorem \ref{thm:rho_smooth} uses an intermediate result which states that neural networks can approximate coordinate projection maps in Sobolev norm under certain conditions on the activation function.

\begin{lem} \label{lem:proj}
	Let $\rho: \R \to \R$ be real analytic, bounded, and not constant. Suppose that all derivatives $\rho^{(n)}$ of $\rho$ are bounded. Then for every $d, L, k \in \N$, $p \in [1,\infty]$, $\epsilon>0$, $B>0$, and every $i \in \{1,\dots,d\}$, one can construct a neural network $\tilde{\Phi}^{B}_{\epsilon,i,k,p} \in \mathcal{NN}(d,1,\dots,1)$ with $L$ layers such that
	\begin{equation*}
	\left\| R_\rho^{[-B,B]^d}(\tilde{\Phi}^{B}_{\epsilon,i,k,p}) - P_i \right\|_{W^{k,p}([-B,B]^d)} \leq \epsilon,
	\end{equation*}
	where $P_i(x) = x_i$. 
\end{lem}

\begin{proof}[Proof of Lemma \ref{lem:proj}.]
	Without loss of generality, we only consider the case $\epsilon \leq 1$. Set $\Omega \defeq [-B,B]^d$ and $\epsilon' \defeq \epsilon/\left((k+1)(2B)^{d/p} L\right)$ (in particular, $\epsilon' = \epsilon/(k+1)L$ for $p=\infty$). Since $\rho$ is in $C^\infty(\R)$ and is not constant, there exists $z_0 \in \R$ such that $\rho'(z_0) \ne 0$. For each $C>0$, define $\tilde{\Phi}^C_1 \defeq \big( (A_1,b_1),(A_2,b_2) \big) \in \mathcal{NN}(d,1,1)$ by
	\begin{equation*}
		A_1 = \begin{pmatrix}
			\tfrac{1}{C}	&0	&\cdots	&0
		\end{pmatrix}	\in \R^{1 \times d},	\hspace{5mm}	b_1 = z_0 \in \R^1,	\hspace{5mm}	A_2 = \frac{C}{\rho'(z_0)} \in \R^{1 \times 1},	\hspace{5mm}	b_2 = -\frac{C \rho(z_0)}{\rho'(z_0)} \in \R^{1}
	\end{equation*} 
	so that 
	\begin{equation*}
		R^\Omega_\rho(\tilde{\Phi}^C_1)(x) = \frac{C}{\rho'(z_0)} \cdot \rho\left( \frac{x_1}{C} + z_0 \right) - \frac{C \rho(z_0)}{\rho'(z_0)}
	\end{equation*}
	where $x=(x_1,\dots,x_d)$. Notice that
	\begin{equation*}
		\lim_{C \to \infty} R^\Omega_\rho(\tilde{\Phi}^C_1)(x) = \lim_{C \to \infty} x_1 \cdot \frac{1}{\rho'(z_0)} \cdot \frac{\rho(z_0+x_1/C)-\rho(z_0)}{x_1/C} = x_1
	\end{equation*}
	pointwise. In fact, there exists some $C_0>0$ such that $| R^\Omega_\rho(\tilde{\Phi}^C_1)(x) - x_1 | \leq \epsilon'$ for all $x \in (-B-L\epsilon,B+L\epsilon)^d$ and all $C \geq C_0$. To see why, notice that by the definition of the derivative there exists $\delta>0$ such that
	\begin{equation}\label{eq:Prop_deriv}
		\left| \frac{\rho(z_0+t) - \rho(z_0)}{t} - \rho'(z_0) \right| \leq \frac{ |\rho'(z_0)| \cdot \epsilon' }{1+B+L}
	\end{equation}
	for all $t \in \R$ with $|t| \leq \delta$. Set $C_0 \defeq (B+L)/\delta$ and let $C \geq C_0$ be arbitrary. Since $\epsilon \leq 1$, every $x \in (-B-L\epsilon,B+L\epsilon)^d$ satisfies $|x_1| \leq B+L$. If we set $t = x_1/C$, then $$|t| = |x_1|/C \leq (B+L)/C \leq (B+L)/C_0 = \delta.$$ It follows that
	\begin{align*}
		| R^\Omega_\rho(\tilde{\Phi}^C_1)(x) - x_1 | &= \left| \frac{C}{\rho'(z_0)} \cdot \rho\left( \frac{x_1}{C} + z_0 \right) - \frac{C \rho(z_0)}{\rho'(z_0)} - x_1 \right|	\\
		&= \left| \frac{C}{\rho'(z_0)} \right| \left| \rho\left( z_0 + \frac{x_1}{C} \right) - \rho(z_0) - \rho'(z_0) \cdot \frac{x_1}{C} \right| \\
		&= \left| \frac{C}{\rho'(z_0)} \right| \left| \rho\left( z_0 + t \right) - \rho(z_0) - \rho'(z_0) t \right| \\
		&\leq \left| \frac{C}{\rho'(z_0)} \right| \cdot \frac{ |\rho'(z_0)| \cdot \epsilon' }{1+B+L} \cdot |t|	&&(\text{by } \eqref{eq:Prop_deriv}) \\
		&= \left| \frac{C}{\rho'(z_0)} \right| \cdot \frac{ |\rho'(z_0)| \cdot \epsilon' }{1+B+L} \cdot \left| \frac{x_1}{C} \right| \\
		&= \frac{|x_1|}{1+B+L} \cdot \epsilon'	\\
		&\leq \epsilon'
	\end{align*}
	for all $x \in (-B-L\epsilon,B+L\epsilon)^d$ and all $C \geq C_0$. We also have
	\begin{equation*}
		\frac{\partial}{\partial x_1} R^\Omega_\rho(\tilde{\Phi}^C_1)(x) = \frac{1}{\rho'(z_0)} \cdot \rho'\left( \frac{x_1}{C} + z_0 \right) \xrightarrow{C \to \infty} 1 
	\end{equation*}
	pointwise. Note that $\rho'$ is Lipschitz because $\rho''$ is bounded. Hence, if we define $C_1 \defeq B\cdot \text{Lip}(\rho')/(|\rho'(z_0)| \cdot \hat{\epsilon}_1)$ (where $\hat{\epsilon}_1>0$ will be chosen later) then
	\begin{align*}
		\left| \frac{\partial}{\partial x_1} R^\Omega_\rho(\tilde{\Phi}^C_1)(x) - 1 \right| &= \left| \frac{1}{\rho'(z_0)} \cdot \rho'\left( \frac{x_1}{C} + z_0 \right) - 1 \right|	\\
		&= \left| \frac{1}{\rho'(z_0)} \right| \left| \rho'\left( z_0+\frac{x_1}{C} \right) - \rho'(z_0) \right| \\
		&\leq \left| \frac{1}{\rho'(z_0)} \right| \text{Lip}(\rho') \left| \frac{x_1}{C} \right| \\
		&\leq \frac{ B \cdot \text{Lip}(\rho') }{ |\rho'(z_0)| } \cdot \frac{1}{C_1} \\
		&= \hat{\epsilon}_1
	\end{align*}
	for all $x \in \Omega$ and all $C \geq C_1$. We also have
	\begin{equation*}
		\frac{\partial^n}{\partial x_1^n} R^\Omega_\rho(\tilde{\Phi}^C_1)(x) = \frac{1}{C^{n-1}} \cdot \frac{1}{\rho'(z_0)} \cdot \rho^{(n)}\left( \frac{x_1}{C} + z_0 \right) \xrightarrow{C \to \infty} 0 
	\end{equation*}
	pointwise for $n \geq 2$. In fact, if we define $C_n \defeq \max(1,\|\rho^{(n)}\|_\infty/(|\rho'(z_0)| \cdot \hat{\epsilon}_n))$ (where $\hat{\epsilon}_n>0$ will be chosen later) then
	\begin{align*}
		\left| \frac{\partial^n}{\partial x_1^n} R^\Omega_\rho(\tilde{\Phi}^C_1)(x) - 0 \right| &= \left| \frac{1}{C^{n-1}} \cdot \frac{1}{\rho'(z_0)} \cdot \rho^{(n)}\left( \frac{x_1}{C} + z_0 \right) - 0 \right| \\
		&= \frac{1}{C^{n-1}} \cdot \frac{1}{|\rho'(z_0)|} \cdot \left| \rho^{(n)}\left( \frac{x_1}{C} + z_0 \right) \right| \\
		&\leq \frac{1}{C}\cdot \frac{1}{|\rho'(z_0)|} \cdot \|\rho^{(n)}\|_\infty \\
		&\leq \frac{1}{C_n}\cdot \frac{1}{|\rho'(z_0)|} \cdot \|\rho^{(n)}\|_\infty \\
		&\leq \hat{\epsilon}_n
	\end{align*}
	for all $x \in \Omega$ and all $C \geq C_n$. Set $C^* = \max(C_0,C_1,\dots,C_k)$. Then
	\begin{align*}
		\left| R^\Omega_\rho(\tilde{\Phi}^{C^*}_1)(x) - x_1 \right| &\leq \epsilon', \\
		\left| \frac{\partial}{\partial x_1} R^\Omega_\rho(\tilde{\Phi}^{C^*}_1)(x) - 1 \right| &\leq \hat{\epsilon}_1, \\
		\left| \frac{\partial^n}{\partial x_1^n} R^\Omega_\rho(\tilde{\Phi}^{C^*}_1)(x) - 0 \right| &\leq \hat{\epsilon}_n, &n=2,\dots,k
	\end{align*}
	for all $x \in \Omega$. 
	Furthermore, any partial derivatives involving $x_2,\dots,x_n$ are identically zero. Next define a network $\tilde{\Phi}^{C^*}_2 = \big( (A_1',b_1'),(A_2',b_2') \big) \in \mathcal{NN}(1,1,1)$ by
	\begin{equation*}
		A_1' = \tfrac{1}{C^*}	\in \R^{1 \times 1},	\hspace{5mm}	b_1' = z_0 \in \R^1,	\hspace{5mm}	A_2' = \frac{C^*}{\rho'(z_0)} \in \R^{1 \times 1},	\hspace{5mm}	b_2' = -\frac{C^* \rho(z_0)}{\rho'(z_0)} \in \R^{1}
	\end{equation*}
	and define $\tilde{\Phi}^{C^*} = \tilde{\Phi}^{C^*}_2 \bullet \cdots \bullet \tilde{\Phi}^{C^*}_2 \bullet \tilde{\Phi}^{C^*}_1 \in \mathcal{NN}(d,1,\dots,1)$, where we take $L-2$ concatenations to get $L$ layers. Then
	\begin{equation*}
		R^\Omega_\rho(\tilde{\Phi}^{C^*})(x) = (\rho_{C^*} \circ \cdots \circ \rho_{C^*})(x_1)
	\end{equation*} 
	where $\rho_{C^*}(z) = \tfrac{C^*}{\rho'(z_0)} \cdot \rho\left( \frac{z}{C^*} + z_0 \right) - \tfrac{C^* \rho(z_0)}{\rho'(z_0)}$ is applied $L$ times. Inductively, we have
	\begin{equation} \label{eq:Leps}
		\left| R^\Omega_\rho(\tilde{\Phi}^{C^*})(x) - x_1 \right| \leq L \epsilon' = \frac{\epsilon}{(k+1)(2B)^{d/p}}
	\end{equation}
	by applying $\rho_C$ $L$ times. Working with the derivatives $\frac{\partial^n}{\partial x_1^n} \big( \rho_C \circ \cdots \circ \rho_C \big)(x_1)$ is not as simple because the derivatives of this composition of functions will involve applications of the chain rule and product rule. However, we only apply the chain rule and product rule finitely many times, so Equation \eqref{eq:Leps} and the boundedness of all derivatives of $\rho$ guarantee that $\hat{\epsilon}_1,\dots,\hat{\epsilon}_k$ can be chosen so that 
	\begin{align*}
		\left| \frac{\partial}{\partial x_1} R^\Omega_\rho(\tilde{\Phi}^{C^*})(x) - 1 \right| &\leq \frac{\epsilon}{(k+1)(2B)^{d/p}}	\\
		\left| \frac{\partial^n}{\partial x_1^n} R^\Omega_\rho(\tilde{\Phi}^{C^*})(x) - 0 \right| &\leq \frac{\epsilon}{(k+1)(2B)^{d/p}}	&&n=2,\dots,k.
	\end{align*}
	For $p<\infty$, it follows that
	\begin{align*}
		\left\| R_\rho^{\Omega}(\tilde{\Phi}^{C^*}) - P_1 \right\|_{W^{k,p}(\Omega)} &= \sum_{|\alpha| \leq k} \left\| D^\alpha R_\rho^{\Omega}(\tilde{\Phi}^{C^*}) - D^\alpha P_1 \right\|_{L^p(\Omega)} \\
		&= \sum_{n=0}^{k} \left\| \frac{\partial^n}{\partial x_1^n} R_\rho^{\Omega}(\tilde{\Phi}^{C^*}) - \frac{\partial^n}{\partial x_1^n} P_1 \right\|_{L^p(\Omega)} \\
		&= \sum_{n=0}^k \left( \int_{\Omega} \left| \frac{\partial^n}{\partial x_1^n} R_\rho^{\Omega}(\tilde{\Phi}^{C^*})(x) - \frac{\partial^n}{\partial x_1^n} P_1(x) \right|^p dx \right)^{1/p} \\
		&\leq \sum_{n=0}^k \left( \int_{\Omega} \left| \frac{\epsilon}{(k+1)(2B)^{d/p} } \right|^p dx \right)^{1/p} \\
		&= \sum_{n=0}^k \left( (2B)^d \cdot \frac{\epsilon^p}{(k+1)^p (2B)^{d}} \right)^{1/p} \\
		&= \sum_{n=0}^k \frac{\epsilon}{k+1} \\
		&= \epsilon.
	\end{align*}
	For $p=\infty$, we have
	\begin{align*}
	    \left\| R_\rho^{\Omega}(\tilde{\Phi}^{C^*}) - P_1 \right\|_{W^{k,\infty}(\Omega)} &= \sum_{|\alpha| \leq k} \left\| D^\alpha R_\rho^{\Omega}(\tilde{\Phi}^{C^*}) - D^\alpha P_1 \right\|_{L^\infty(\Omega)} \\
	    &= \sum_{n=0}^{k} \left\| \frac{\partial^n}{\partial x_1^n} R_\rho^{\Omega}(\tilde{\Phi}^{C^*}) - \frac{\partial^n}{\partial x_1^n} P_1 \right\|_{L^\infty(\Omega)} \\
	    &\leq \sum_{n=0}^k \frac{\epsilon}{k+1} \\
		&= \epsilon.
	\end{align*}
	For $i=2,\dots,n$ we can just permute the columns of $A_1$ accordingly.
\end{proof}

We will also need to consider Sobolev convergence of compositions of functions, for which the following lemma will be useful.

\begin{lem} \label{lem:SobolevComposition}
	Let $k \in \N$ and $p \in [1,\infty]$, and let $\Omega \subset \R^d$ be bounded and compact. Suppose $(g_n)_{n \in \N} \subset C^\infty(\Omega)$ is a sequence of functions such that $\|g_n - h\|_{W^{k,p}} \to 0$ for some $h \in C^\infty(\Omega)$. 
	If $(f_n)_{n \in \N} \subset C^\infty(\Omega)$ is a sequence of functions whose derivatives are all bounded independent of $n$, then
	\begin{equation*}
		\lim_{n \to \infty} \|f_n \circ g_n - f_n \circ h\|_{W^{k,p}(\Omega)} = 0.
	\end{equation*}
	That is, if $g_n \to h$ in Sobolev norm, then $f_n \circ g_n \to f_n \circ h$ in Sobolev norm.
\end{lem}

\begin{proof}[Proof of Lemma \ref{lem:SobolevComposition}.]
	For any $\ell \leq k$, we have
	\begin{equation*}
		\frac{\partial^l}{\partial x_1 \cdots \partial x_l} (f_n \circ g_n)(x) = \sum_{\pi \in \Pi} f_n^{(|\pi|)}(g_n(x)) \cdot \prod_{B \in \pi} \frac{ \partial^{|B|} g_n(x) }{ \prod_{j \in B} \partial x_j }
	\end{equation*}
	and
	\begin{equation*}
		\frac{\partial^l}{\partial x_1 \cdots \partial x_l} (f_n \circ h)(x) = \sum_{\pi \in \Pi} f_n^{(|\pi|)}(h(x)) \cdot \prod_{B \in \pi} \frac{ \partial^{|B|} h(x) }{ \prod_{j \in B} \partial x_j }
	\end{equation*}
	by Fa\`{a} di Bruno's formula, where the sum is taken over all set partitions $\pi$ of the set $\{1,\dots,l\}$ and the product is taken over all blocks $B$ in the partition $\pi$. Subtracting these two expressions, we get a sum over $\pi \in \Pi$ of terms of the form
	\begin{align}
		&f_n^{(|\pi|)}(g_n(x)) \cdot \prod_{B \in \pi} \frac{ \partial^{|B|} g_n(x) }{ \prod_{j \in B} \partial x_j } - f_n^{(|\pi|)}(h(x)) \cdot \prod_{B \in \pi} \frac{ \partial^{|B|} h(x) }{ \prod_{j \in B} \partial x_j } \notag \\[+2mm]
		&\hspace{1cm} = f_n^{(|\pi|)}(g_n(x)) \cdot \prod_{B \in \pi} \frac{ \partial^{|B|} g_n(x) }{ \prod_{j \in B} \partial x_j } - f_n^{(|\pi|)}(g_n(x)) \cdot \prod_{B \in \pi} \frac{ \partial^{|B|} h(x) }{ \prod_{j \in B} \partial x_j } \label{eq:first_term} \\
		&\hspace{2cm} + f_n^{(|\pi|)}(g_n(x)) \cdot \prod_{B \in \pi} \frac{ \partial^{|B|} h(x) }{ \prod_{j \in B} \partial x_j } - f_n^{(|\pi|)}(h(x)) \cdot \prod_{B \in \pi} \frac{ \partial^{|B|} h(x) }{ \prod_{j \in B} \partial x_j } \label{eq:second_term} 
	\end{align}
	(note that we added and subtracted a cross term). For the first term \eqref{eq:first_term}, we have
	\begin{align*}
		&\left\| f_n^{(|\pi|)}(g_n(\cdot)) \prod_{B \in \pi} \frac{ \partial^{|B|} g_n(\cdot) }{ \prod_{j \in B} \partial x_j } - f_n^{(|\pi|)}(g_n(\cdot)) \prod_{B \in \pi} \frac{ \partial^{|B|} h(\cdot) }{ \prod_{j \in B} \partial x_j } \right\|_{L^p(\Omega)} \\
		&\hspace{1cm} \leq C \left\| \prod_{B \in \pi} \frac{ \partial^{|B|} g_n(\cdot) }{ \prod_{j \in B} \partial x_j } - \prod_{B \in \pi} \frac{ \partial^{|B|} h(\cdot) }{ \prod_{j \in B} \partial x_j } \right\|_{L^p(\Omega)}
	\end{align*}
	for some constant $C$ since $f_n^{(|\pi|)}$ is bounded independent of $n$. Note that for each factor in the product over $B \in \pi$, we have
	\begin{equation*}
		\lim_{n \to \infty} \left\| \frac{ \partial^{|B|} g_n(\cdot) }{ \prod_{j \in B} \partial x_j } - \frac{ \partial^{|B|} h(\cdot) }{ \prod_{j \in B} \partial x_j } \right\|_{L^p(\Omega)} = 0
	\end{equation*}
	since $\lim_{n \to \infty} \|g_n - h\|_{W^{k,p}(\Omega)} = 0$ and $|B| \leq \ell \leq k$. It follows that the entire product converges to 0 in $L^p$ norm as $n \to \infty$. Indeed, for $\phi_1, \phi_2, \psi_1, \psi_2$ bounded and $p<\infty$,
	\begin{align*}
		\|\phi_1\phi_2 - \psi_1\psi_2\|_{L^p(\Omega)}^p &= \int_{\Omega} |\phi_1(x)\phi_2(x) - \psi_1(x)\psi_2(x)|^p dx \\
		&= \int_{\Omega} |\phi_1(x)\phi_2(x) - \psi_1(x)\phi_2(x) + \psi_1(x)\phi_2(x) - \psi_1(x)\psi_2(x)|^p dx \\
		&\leq 2^{p-1} \left( \int_{\Omega} |\phi_1(x)\phi_2(x) - \psi_1(x)\phi_2(x)|^p dx + \int_{\Omega} |\psi_1(x)\phi_2(x) - \psi_1(x)\psi_2(x)|^p dx \right) \\
		&\leq 2^{p-1} C_{\psi_1,\phi_2} \left( \|\phi_1-\psi_1\|_{L^p(\Omega)}^p + \|\phi_2-\psi_2\|_{L^p(\Omega)}^p \right)
	\end{align*}
	by Minkowski's inequality, where $C_{\psi_1,\phi_2}$ bounds both $|\psi_1|^p$ and $|\phi_2|^p$ on $\Omega$. 
	For $p=\infty$,
	\begin{align*}
	    \|\phi_1\phi_2 - \psi_1\psi_2\|_{L^\infty(\Omega)} &= \sup_{x \in \Omega} |\phi_1(x)\phi_2(x) - \psi_1(x)\psi_2(x)| \\
	    &= \sup_{x \in \Omega} |\phi_1(x)\phi_2(x) - \psi_1(x)\phi_2(x) + \psi_1(x)\phi_2(x) - \psi_1(x)\psi_2(x)| \\
	    &\leq \sup_{x \in \Omega} \left( |\phi_2(x)| |\phi_1(x)-\psi_1(x)| + |\psi_1(x)| |\phi_2(x)-\psi_2(x)| \right) \\
	    &\leq C_{\psi_1,\phi_2} \left( \sup_{x \in \Omega} |\phi_1(x)-\psi_1(x)| + \sup_{x \in \Omega} |\phi_2(x)-\psi_2(x)| \right) \\
	    &= C_{\psi_1,\phi_2} \left( \|\phi_1-\psi_1\|_{L^\infty(\Omega)} + \|\phi_2-\psi_2\|_{L^\infty(\Omega)} \right)
	\end{align*}
	where $C_{\psi_1,\phi_2}$ bounds both $|\psi_1|$ and $|\phi_2|$ on $\Omega$. Thus, if $\phi_{1,n} \to \psi_1$ and $\phi_{2,n} \to \psi_2$ in $L^p$ norm as $n \to \infty$ with $C_{\psi_1,\phi_2}$ independent of $n$, then $\phi_{1,n}\phi_{2,n} \to \psi_1\psi_2$ in $L^p$ norm. For the second term \eqref{eq:second_term}, we have
	\begin{align*}
		&\left\| f_n^{(|\pi|)}(g_n(\cdot)) \prod_{B \in \pi} \frac{ \partial^{|B|} h(\cdot) }{ \prod_{j \in B} \partial x_j } - f_n^{(|\pi|)}(h(\cdot)) \prod_{B \in \pi} \frac{ \partial^{|B|} h(\cdot) }{ \prod_{j \in B} \partial x_j } \right\|_{L^p(\Omega)} \\
		&\hspace{1cm} \leq C_h \left\|f_n^{(|\pi|)}(g_n(\cdot)) - f_n^{(|\pi|)}(h(\cdot)) \right\|_{L^p(\Omega)}
	\end{align*}
	for some constant $C_h$ since all derivatives of $h$ are bounded because $h \in C^\infty(\Omega)$ with $\Omega$ compact. Since all derivatives of $f_n$ are bounded independent of $n$, $f_n^{(|\pi|)}$ is Lipschitz with constant $M$ independent of $n$. Thus, for $p<\infty$, we have
	\begin{align*}
		\left\|f_n^{(|\pi|)}(g_n(\cdot)) - f_n^{(|\pi|)}(h(\cdot)) \right\|_{L^p(\Omega)}^p &= \int_{\Omega} |f_n^{(|\pi|)}(g_n(x)) - f_n^{(|\pi|)}(h(x))|^p dx \\
		&\leq M^p \int_{\Omega} |g_n(x) - h(x)|^p dx \\
		&= M^p \|g_n - h\|_{L^p(\Omega)}^p.
	\end{align*}
	For $p=\infty$, we have
	\begin{align*}
	    \left\|f_n^{(|\pi|)}(g_n(\cdot)) - f_n^{(|\pi|)}(h(\cdot)) \right\|_{L^\infty(\Omega)} &= \sup_{x \in \Omega} \big|f_n^{(|\pi|)}(g_n(x)) - f_n^{(|\pi|)}(h(x)) \big| \\
	    &\leq M \sup_{x \in \Omega} |g_n(x)-h(x)| \\
	    &= M \|g_n-h\|_{L^\infty(\Omega)}.
	\end{align*}
	In either case, the second term converges to 0 as $n \to \infty$. It follows that
	\begin{equation*}
		\lim_{n \to \infty} \left\| D^\alpha \big( f_n \circ g_n - f_n \circ h \big) \right\|_{L^p(\Omega)} = 0
	\end{equation*}
	for any multi-index $\alpha$ with $|\alpha| \leq k$, so $\lim_{n \to \infty} \| f_n \circ g_n - f_n \circ h \|_{W^{k,p}(\Omega)} = 0$.
\end{proof}

\begin{proof}[Proof of Theorem \ref{thm:rho_smooth}] Set $\Omega \defeq [-B,B]^d$. Since $\rho \in C^\infty(\R)$ is not constant, there exists $z_0 \in \R$ such that $\rho'(z_0) \ne 0$. For each $n \in N$, define $\Phi_n^1 \defeq \big( (A_1^n,b_1^n),(A_2^n,b_2^n) \big) \in \mathcal{NN}(1,2,1)$ by
\begin{equation*}
	A_1^n = \begin{pmatrix}
		1	\\
		1/n
	\end{pmatrix} \in \R^{2 \times 1}, \hspace{5mm} b_1^n = \begin{pmatrix}
		0	\\
		z_0
	\end{pmatrix} \in \R^2, \hspace{5mm} A_2^n = \begin{pmatrix}
		1	&n
	\end{pmatrix} \in \R^{1 \times 2}, \hspace{5mm} b_2^n = -n\rho(z_0) \in \R^1
\end{equation*}
so that
\begin{equation*}
	R_\rho^{\R}(\Phi_n^1)(x) = \rho(x) + n \rho(x/n+z_0) - n\rho(z_0) = \rho(x) + x \cdot \frac{\rho(z_0+x/n)-\rho(z_0)}{x/n}
\end{equation*}
for all $x \in \R$. By Lemma \ref{lem:proj}, there exists a sequence of neural networks $(\Phi_n^2)_{n \in \N} \subset \mathcal{NN}(d,1,\dots,1)$ with $L-1$ layers such that
\begin{equation}\label{eq:approx_proj}
	\left\| R_\rho^\Omega(\Phi_n^2) - P_1 \right\|_{W^{k,p}(\Omega)} \leq \frac{1}{n}
\end{equation} 
for each $n \in \N$. Define $\Phi_n = \Phi^1_n \bullet \Phi^2_n \in \mathcal{NN}(d,1,\dots,1,2,1)$. Consider the map $F: \R^d \to \R$ defined by $F(x) = \rho(x_1) + \rho'(z_0) x_1$. We have
\begin{align*}
	\left\| R_\rho^\Omega(\Phi_n) - F \right\|_{W^{k,p}((-B,B)^d)} &= \left\| R_\rho^\Omega(\Phi_n) - R_\rho^\R(\Phi_n^1) \circ P_1 + R_\rho^\R(\Phi_n^1) \circ P_1 - F \right\|_{W^{k,p}(\Omega)} \\
	&\leq \left\| R_\rho^\R(\Phi_n^1) \circ R_\rho^\Omega(\Phi_n^2) - R_\rho^\R(\Phi_n^1) \circ P_1 \right\|_{W^{k,p}(\Omega)} \\ 
	&\quad \quad + \left\| R_\rho^\R(\Phi_n^1) \circ P_1 - F \right\|_{W^{k,p}(\Omega)}.
\end{align*}
Since $\tfrac{d^\ell}{dx^\ell} R_\rho^\R(\Phi_n^1)(x) = \rho^{(\ell)}(x) + \rho^{(\ell)}(x/n+z_0)/n^{\ell-1}$, each derivative of $R_\rho^\R(\Phi_n^1)$ is bounded independent of $n$. Combined with equation \eqref{eq:approx_proj} and the boundedness of all derivatives of $\rho$, this implies that
\begin{equation*}
	\left\| R_\rho^\R(\Phi_n^1) \circ R_\rho^\Omega(\Phi_n^2) - R_\rho^\R(\Phi_n^1) \circ P_1 \right\|_{W^{k,p}(\Omega)} \xrightarrow{n \to \infty} 0
\end{equation*} 
by Lemma \ref{lem:SobolevComposition}. Next note that
\begin{equation*}
	R_\rho^\R(\Phi_n^1) \circ P_1(x) = \rho(x_1) + x_1 \cdot \frac{\rho(z_0+x_1/n)-\rho(z_0)}{x_1/n} \xrightarrow{n \to \infty} \rho(x_1) + \rho'(z_0) x_1 = F(x)
\end{equation*}
pointwise. An argument similar to the one in the proof of Theorem \ref{thm:m-(m+1)} shows that $$\left\| R_\rho^\R(\Phi_n^1) \circ P_1 - F \right\|_{W^{k,p}(\Omega)} \xrightarrow{n \to \infty} 0$$ since the pointwise derivative agrees with the strong $L^p$ derivative. In total, we have $$\left\| R_\rho^\Omega(\Phi_n) - F \right\|_{W^{k,p}(\Omega)} \xrightarrow{n \to \infty} 0.$$ However, $F|_\Omega \notin \mathcal{RNN}(d,1,\dots,1,2,1)$. To see why suppose we had $F|_\Omega = R_\rho^\Omega(\Psi)$ for some neural network $\Psi$. Since $\rho$ is analytic, $F$ and $R_\rho^{\R^d}(\Psi)$ are both analytic 
and coincide on $\Omega = [-B,B]^d$. Hence, we must have $F = R_\rho^{\R^d}(\Psi)$ on all of $\R^d$. Note that $F(x) = \rho(x_1) + \rho'(z_0)x_1$ is unbounded because $\rho$ is bounded and $\rho'(z_0) \ne 0$. However, $R_\rho^{\R^d}(\Psi)$ is bounded because $\rho$ is. This contradicts $F = R_\rho^{\R^d}(\Psi)$, so $F|_\Omega \ne R_\rho^\Omega(\Psi)$. Since $$\left\| R_\rho^\Omega(\Phi_n) - F \right\|_{W^{k,p}((-B,B)^d)} \xrightarrow{n \to \infty} 0$$ with $R_\rho^\Omega(\Phi_n) \in \mathcal{RNN}(d,1,\dots,1,2,1)$ but $F|_\Omega \notin \mathcal{RNN}(d,1,\dots,1,2,1)$, we have shown that $\mathcal{RNN}(d,1,\dots,1,2,1)$ is not closed.
\end{proof}

\subsection{Proof of Proposition \ref{prop:rates}} \label{subapp:rates}
Let $\rho: \R \rightarrow \R$ be arbitrary, and define $(h_n)_{n=1}^\infty = (R_\rho^\R(\Phi_n))_{n=1}^\infty$ to be the realized networks in $\mathcal{RNN}(1,2,1)$ as in the proof of Theorem \ref{thm:m-(m+1)} (even if $\rho$ does not satisfy the hypotheses of the theorem). Then
\[
\norm{\Phi_n}_{total} = \max\{1,n\} + \max\{0,1/n\} = n + (1/n) < 2n.
\]
Thus, our goal will be to prove for compact, measurable $\Omega \subset \R$ with non-empty interior and various $\rho$ that there exists $C_p$, dependent on $\rho$, $p$, and possibly $\Omega$ but independent of $n$ such that
\[
\norm{h_n-\rho'}^p_{L^p(\Omega)} \leq C_p n^{-p}.
\]
For $p=\infty$, we will show directly that
\[
\norm{h_n-\rho'}_{L^\infty(\Omega)} \leq C n^{-1}
\]
for some $C$ independent of $n$. For some of the $\rho$, the bound will be over all of $\R$ rather than over $\Omega$.

Let $\rho \in C^m(\R)$ for some $m \geq 2$.  Since $\Omega$ is compact, $\norm{\rho''}_{L^\infty(\widetilde{\Omega})}$ is finite, where $\widetilde{\Omega} = \{ x \, : \, x \enskip \textrm{or} \enskip x-(1/n) \in \Omega\}$. Let $x \in \Omega$. It follows from a first order Taylor approximation that there exists $\xi \in (x, x+(1/n))$ such that
\begin{align*}
    \abs{h_n(x)-\rho'(x)} &= \abs{\frac{\rho(x+(1/n))-\rho(x)}{1/n} - \rho'(x)} =\frac{1}{2n} \abs{\rho''(\xi)} \leq \frac{1}{2n} \norm{\rho''}_{L^\infty(\widetilde{\Omega})}.
\end{align*}
Thus, for $p<\infty$, we have
\[
\norm{h_n-\rho'}^p_{L^p(\Omega)}\leq \left(\frac{1}{2n} \norm{\rho''}_{L^\infty(\widetilde{\Omega})}\right)^p \abs{\Omega} \eqdef C_p n^{-p}
\]
where $\abs{\Omega}$ is the measure of $\Omega$ and $C_p$ is independent of $n$. For $p=\infty$, we have
\[
\norm{h_n-\rho'}_{L^\infty(\Omega)} \leq \frac{1}{2n} \norm{\rho''}_{L^\infty(\widetilde{\Omega})} \eqdef Cn^{-1}
\]
where $C$ is independent of $n$.

Now let $\rho$ be the softsign function $\rho(x)=x/(1+\abs{x})$ from Table \ref{table:activation}.  Then $\rho'(x)=1/(1+\abs{x})^2$. It follows that
\begin{align*}
    \abs{h_n(x)-\rho'(x)} &= \abs{\frac{\rho(x+(1/n))-\rho(x)}{1/n} - \rho'(x)} \\[2mm]
    &= \abs{n\left(\frac{x+(1/n)}{1+\abs{x+(1/n)}}-\frac{x}{1+\abs{x}} \right)-\frac{1}{\left(1+\abs{x}\right)^2}}\\[2mm]
    &= \left\{\begin{array}{lr}
        \frac{1}{n(1+x+(1/n))(1+x)^2} & x \geq 0  \\[2mm]
        \frac{\abs{2nx^3+2(1-n)x^2-4x-(1/n)}}{(1+x+(1/n))(1-x)^2} & -1/n \leq x < 0 \\[2mm]
        \frac{1}{n(1-x-(1/n))(1-x)^2} & x < -1/n.
    \end{array} \right.
\end{align*}
For $p<\infty$, we have
\begin{align*}
\lefteqn{\norm{h_n-\rho'}_{L^p(\Omega)}^p }\\ 
&= \int_0^\infty\left(\frac{1}{n(1+x+(1/n))(1+x)^2}\right)^p dx + \int_{-1/n}^0 \left(\frac{\abs{2nx^3+2(1-n)x^2-4x-(1/n)}}{(1+x+(1/n))(1-x)^2}\right)^p dx \\
&\hspace{4mm}+ \int_{-\infty}^{-1/n} \left(\frac{1}{n(1-x-(1/n))(1-x)^2}\right)^p dx\\
&\eqdef A + B + C.
\end{align*}
We begin by bounding $n^p A$:
\begin{equation}
n^p A = \int_0^\infty\frac{1}{(1+x+(1/n))^p(1+x)^{2p}} dx < \int_0^\infty\frac{1}{(1+x)^{3p}} dx = \frac{1}{3p-1}. \label{eqn:abnd}
\end{equation}
Similarly,
\begin{equation}
n^p C = \int_{-\infty}^{-1/n} \frac{1}{(1-x-(1/n))^p(1-x)^{2p}} dx < \int_{-\infty}^{-1/n} \frac{1}{(1-x-(1/n))^{3p}} dx = \frac{1}{3p-1}.
\label{eqn:cbnd}
\end{equation}
We now seek to bound $B$. Let $q(x)=2nx^3+2(1-n)x^2-4x-(1/n)$ so that $\abs{q(x)}^p$ is the numerator of the integrand of $B$.  We would like to bound $\abs{q(x)}$ over $[-1/n,0]$.  Note that $q(0)=-1/n$ and $q(-1/n)=1/n$.  Also,
\[
q'(x) = 6nx^2 + 4(1-n)x -4,
\]
and thus $q'(x)$ is zero when 
\[
x = \frac{n-1 \pm \sqrt{1 + 4n + n^2}}{3n}.
\]
Since $(n-1 + \sqrt{1 + 4n + n^2})/(3n) > 2/3 >0$, we only care about the case $x_0 = (n-1 - \sqrt{1 + 4n + n^2})/(3n)$.  We claim that $q(x_0) > 0$ and thus an upper bound on $q(x_0)$ is an upper bound on $\abs{q(x_0)}$. To that end, we note that
\[
\frac{-1}{n} = \frac{n-1-\sqrt{4+4n+n^2}}{3n} \leq x_0 \leq \frac{n-1-\sqrt{1+2n+n^2}}{3n} = \frac{-2}{3n},
\]
where $q(-1/n)=1/n$ and $q(-2/(3n))=2/(9n^2) + 7/(9n)$.  Since $q''$ is negative over the interval $[-1/n,0)$, this implies that $q$ is positive over the interval $[-1/n,-2/(3n)]$, which includes $x_0$.  So we now bound $q(x_0)$:
\begin{align*}
    27n^2 q(x_0) &= -4n^3-24n^2-3n+4+(4n^2+16n+4)\sqrt{n^2+4n+1}\\
    &\leq -4n^3-24n^2-3n+4+(4n^2+16n+4)\sqrt{n^2+4n+4}\\
    &= 32n+14.
\end{align*}
As
\[
\frac{64}{27n} >\frac{32n}{27n^2}\left(1+\frac{14}{32n}\right) = \frac{32n+14}{27n^2} > \frac{1}{n},
\]
it is an upper bound for $\abs{q(x)}$ over $[-1/n,0]$. We now bound $B$:
\begin{align}
 \left(\frac{27n}{64}\right)^p  B &\leq  \int_{-1/n}^0 \frac{1}{\left((1+x+(1/n))(1-x)^2\right)^p} dx \nonumber \\
 &\leq  \int_{-1/n}^{-1/(2n)} \frac{1}{(1+x+(1/n))^{3p}} dx + \int_{-1/(2n)}^0 \frac{1}{(1-x)^{3p}} dx \nonumber\\
 &=2 \int_1^{1+1/(2n)}\frac{1}{x^{3p}}dx = \frac{2}{3p-1}\left(1-\left(1+\frac{1}{2n}\right)^{1-3p}\right) \nonumber \\
 &\leq \frac{2}{3p-1} \label{eqn:bbnd}.
\end{align}
Putting together~\eqref{eqn:abnd},~\eqref{eqn:cbnd}, and~\eqref{eqn:bbnd}, we obtain for any measurable $\Omega \subseteq \mathbb{R}$
\begin{align*}
 \norm{h_n-\rho'}^p_{L^p(\Omega)}\leq \norm{h_n-\rho'}^p_{L^p(\mathbb{R})} \leq \frac{1}{(3p-1)} \left( \frac{1}{n^p} + \frac{1}{n^p} +  2\left(\frac{64}{27n}\right)^p\right) \eqdef C_p n^{-p},
\end{align*}
where $C_p$ does not depend on $n$. For $p=\infty$, we can use the piecewise expression for $|h_n(x)-\rho'(x)|$ and the bound for $|q(x)|$ to see that
\[
    |h_n(x)-\rho'(x)| \leq \left\{\begin{array}{lr}
        \frac{1}{n} & x \geq 0  \\[2mm]
        \frac{64}{27n} & -1/n \leq x < 0 \\[2mm]
        \frac{1}{n} & x < -1/n.
    \end{array} \right.
\]
Thus, we have
\[
    \norm{h_n-\rho'}_{L^\infty(\Omega)} \leq \frac{64}{27n} \eqdef Cn^{-1}
\]
where $C$ is independent of $n$.

Finally, let $\rho$ be the ELU function $\rho(x) = x \chi_{x \geq 0} + (e^x-1)\chi_{x < 0}$ from Table \ref{table:activation}.  Then $\rho'(x) = \chi_{x \geq 0} + e^x\chi_{x < 0}$ and for all $x \neq 0$, $\rho''(x) = e^x\chi_{x < 0}$. 

We first compute for $x \geq 0$:
\begin{equation} \label{eq:ELUa}
 \abs{h_n(x)-\rho'(x)} =  \abs{\frac{\rho(x+(1/n))-\rho(x)}{1/n} - \rho'(x)} =  \abs{\frac{x+(1/n)-x}{1/n} - 1}=0.
\end{equation}

For $x < -1/n$, we may use Taylor approximation to prove the existence of a $\xi \in (x, x+(1/n))$ such that
\begin{align*}
 \abs{h_n(x)-\rho'(x)} &= \abs{\frac{\rho(x+(1/n))-\rho(x)}{1/n} - \rho'(x)} = \frac{1}{2n} \abs{\rho''(\xi)} = \frac{1}{2n} e^\xi \leq \frac{e^{1/n}}{2n} e^x, \end{align*}
where for $p<\infty$,
\begin{align}
\norm{\frac{e^{1/n}}{2n} e^\cdot}^p_{L^p((-\infty,-1/n])} &= \int_{-\infty}^{-1/n} \left(\frac{e^{1/n}}{2n} e^x \right)^p dx = \left( \frac{e^{1/n}}{2n}\right)^p \int_{-\infty}^{-1/n} e^{px} dx = \left( \frac{e^{1/n}}{2n}\right)^p\frac{1}{p} e^{-p/n} \nonumber\\
&= \frac{1}{2^p p} \frac{1}{n^p}. \label{eq:ELUb}
\end{align}
What remains is to analyze $ \abs{h_n(x)-\rho'(x)}$ for $-1/n \leq x < 0$.  To this end, we compute
\begin{align*}
  \abs{h_n(x)-\rho'(x)} &= \abs{\frac{\rho(x+(1/n))-\rho(x)}{1/n} - \rho'(x)} = \abs{\frac{(x+(1/n))-(e^x-1)}{1/n} - e^x} \\[2mm]
  &= \abs{nx +1-ne^x + n-e^x} = \abs{(n+1)(1-e^x)+nx} \\
  &\eqdef \abs{g(x)}.
\end{align*}
We will use Taylor's theorem over $[-1/n,0)$ both to show that $g(x)$ is non-negative and to give an upper bound. To this end, for each $x \in [-1/n, 0)$, there exists an $\xi \in (x,0)$ such that
\[
g(x) = (n+1)\left( -x - \frac{e^\xi}{2}x^2\right) + nx = -x -(n+1) \frac{e^\xi}{2} x^2 = -x\left(1 +(n+1) \frac{e^\xi}{2} x\right).
\] 
Since $x < 0$, $-x > 0$ and $1 + (n+1)e^\xi x/2 < 1$.  Thus $-x$ is an upper bound of $g$ over $[-1/n,0)$.  All that is left to show is $1 + (n+1)e^\xi x/2 > 0$.  As $e^{\xi} < e^0 = 1$,
\[
0 < 1 + \frac{n+1}{2} \cdot \frac{-1}{n} < 1 + \frac{(n+1) e^\xi}{2} x,
\]
as desired. This yields
\begin{equation}\label{eq:ELUc}
\norm{g}_{L^p([-1/n,0)}^p = \int_{-1/n}^0 \abs{g(x)}^p dx \leq \int_{-1/n}^0 \abs{x}^p dx = \frac{1}{(p+1)n^{p+1}} < \frac{1}{(p+1)n^p}.
\end{equation}
Putting~\eqref{eq:ELUa},~\eqref{eq:ELUb}, and~\eqref{eq:ELUc} together, we obtain
\[
\norm{h_n-\rho'}_{L^p(\Omega)}^p  <  \frac{1}{(p+1)n^p} + \frac{1}{2^p p} \frac{1}{n^p} \eqdef C_p n^{-p},
\]
where $C_p$ does not depend on $n$. For $p=\infty$, we have seen that $|h_n(x)-\rho'(x)| \leq \frac{1}{n}$
for all $x$, and therefore $\norm{h_n-\rho'}_{L^\infty(\Omega)} \leq n^{-1}$ as desired.


\section{More Experimental Results} \label{app:exp}

In this section, we include experimental results using the ISRLU activation function. Since ISRLU is $C^2$ but not $C^3$, we train the networks to learn randomly generated piecewise quadratic ($C^1$ but not $C^2$) functions in order-2 Sobolev space.

\begin{figure}[ht]
	\centering
	\begin{subfigure}[b]{0.32\linewidth}
	    \centering
	    \includegraphics[scale=0.33]{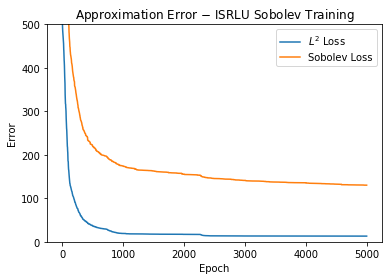}
	    \caption{Training error. \label{subcap:ISRLU_Loss_S}}
	\end{subfigure}
	\begin{subfigure}[b]{0.32\linewidth}
	    \centering
	    \includegraphics[scale=0.33]{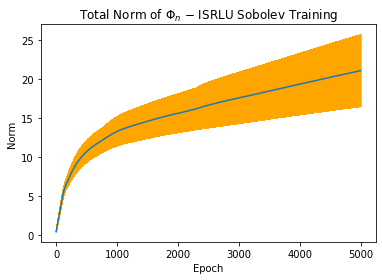}
	    \caption{The network norm. \label{subcap:ISRLU_Norm_S}}
	\end{subfigure}
	\begin{subfigure}[b]{0.32\linewidth}
	    \centering
	    \includegraphics[scale=0.33]{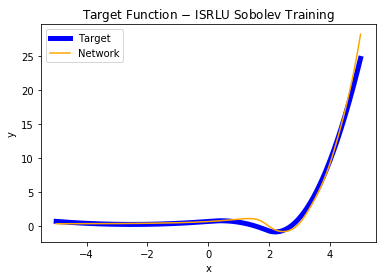}
	    \caption{The target function. \label{subcap:ISRLU_Target_S}}
	\end{subfigure}

	\captionsetup{width=0.85\linewidth}
	\caption{For each of 100 target functions $f$, we train an ISRLU network $\Phi$ in $\mathcal{NN}(1,10,1)$ to minimize $\|R_\rho^{[-5,5]}(\Phi)-f\|_{W^{2,2}}$. (\subref{subcap:ISRLU_Loss_S}) The best $W^{2,2}$ training loss achieved thus far is plotted at each epoch and averaged over all 100 experiments. (\subref{subcap:ISRLU_Norm_S}) The total network norm is averaged over all 100 experiments with 95\% confidence bands. (\subref{subcap:ISRLU_Target_S}) An example target function is plotted along with the realized neural network after training.}\label{cap:ISRLU_S}
\end{figure}

Figure \ref{cap:ISRLU_S} shows the results of 100 trails of ISRLU networks learning non-network target functions in Sobolev norm. We see that the $L^2$ and Sobolev approximation errors decrease rather quickly, though the Sobolev error is still decreasing, indicating that the networks are still learning the target functions. As with the ELU experiments, the total network norm increases quickly at first and then at a fairly steady rate. In this case, however, the 95\% confidence bands are wider because the norm grew fairly large in a few cases and did not grow much in some other cases.

These results are again consistent with Theorem \ref{thm:m-(m+1)}. We see that networks are able to closely approximate non-network target functions, which is evidence of the nonclosedness of sets of realized neural networks in Sobolev spaces. These experiments demonstrate nonclosedness for a $C^2$ activation function in order-2 Sobolev space. Combining this with the ELU and sigmoid results, we see that these numerical observations of nonclosedness and parameter growth hold for varying degrees of smoothness of the activation function.

\end{document}